\documentclass[11pt]{article}

\usepackage{fullpage}
\usepackage{amssymb, amsmath, amsthm}

\usepackage{bm}
\usepackage{enumerate}
\usepackage{algorithm}
\usepackage{graphicx}
\usepackage{framed}
\usepackage{booktabs}
\usepackage{capt-of}
\usepackage{color}

\theoremstyle{plain}
\newtheorem{thm}{Theorem}
\newtheorem{lem}{Lemma}
\newtheorem{cor}{Corollary}

\newtheorem{asm}{Assumption}
\newtheorem{prob}{Problem}

\newtheorem{remark}{Remark}

\def\zero{\bm{0}}
\def\one{\bm{1}}

\def\a{\bm{a}}

\def\e{\bm{e}}
\def\f{\bm{f}}
\def\g{\bm{g}}

\def\n{\bm{n}}

\def\p{\bm{p}}
\def\q{\bm{q}}

\def\v{\bm{v}}
\def\w{\bm{w}}
\def\x{\bm{x}}

\def\z{\bm{z}}

\def\A{\bm{A}}
\def\B{\bm{B}}

\def\F{\bm{F}}
\def\G{\bm{G}}
\def\H{\bm{H}}
\def\I{\bm{I}}

\def\N{\bm{N}}

\def\S{\bm{S}}

\def\V{\bm{V}}
\def\W{\bm{W}}
\def\X{\bm{X}}
\def\Y{\bm{Y}}

\def\Pib{\bm{\Pi}}
\def\Sigmab{\bm{\Sigma}}

\def\FC{\mathcal{F}}
\def\GC{\mathcal{G}}
\def\AC{\mathcal{A}}
\def\BC{\mathcal{B}}

\def\Real{\mathbb{R}}
\def\equivSym{\Leftrightarrow}
\def\trans{\top}
\def\trace{\mathrm{tr}}

\def\diag{\mathrm{diag}}
\def\mmin{\mathrm{min}}
\def\mmax{\mathrm{max}}
\def\out{\mathrm{out}}

\def\OurModel{\mathsf{P}}
\def\BRRT{\mathsf{Q}}
\def\GL{\mathsf{R}}

\def\OptSol{\X_\mathrm{opt}}
\def\OptVal{\theta}

\def\SCORE{\mathrm{score}}
\def\diam{\mathrm{diam}}
\def\error{\mathrm{error}}

\newcommand{\by}[2]{\ensuremath{#1 \times #2}}

\title{Refinement of Hottopixx Method for Nonnegative Matrix Factorization Under Noisy Separability}

\author{Tomohiko~Mizutani
\thanks{Department of Mathematical and Systems Engineering,
Shizuoka University,
3-5-1 Johoku, Naka, Hamamatsu, 432-8561, Japan.
{\tt mizutani.t@shizuoka.ac.jp}}}

\date{\today}

\begin{document}

\maketitle

\begin{abstract}
Hottopixx, proposed by Bittorf et al.\ at NIPS 2012, 
is an algorithm for solving nonnegative matrix factorization (NMF) problems under the separability assumption.
Separable NMFs have important applications, such as topic extraction from documents and unmixing of hyperspectral images.
In such applications,  the robustness of the algorithm to noise is the key to the success.
Hottopixx has been shown to be robust to noise, 
and its robustness can be further enhanced through postprocessing.
However, there is a drawback.
Hottopixx and its postprocessing require us to estimate the noise level involved in the matrix we want to factorize before running,
since they use it as part of the input data.
The noise-level estimation is not an easy task.
In this paper, we overcome this drawback.  We present a refinement of Hottopixx and its postprocessing
that runs without prior knowledge of the noise level.
We show that the refinement has almost the same robustness to noise as the original algorithm.

\end{abstract}

\section{Introduction} \label{Sec: Introduction}
Let $\Real_+^{d \times n}$ denote the set of all nonnegative matrices of size \by{d}{n}.
We are given $\V \in \Real_+^{d \times n}$ and the factorization rank $r$.
The nonnegative matrix factorization (NMF) problem asks us to find the factors 
$\W \in \Real_+^{d \times r}$ and $\H \in \Real_+^{r \times n}$ of $\V$ minimizing the gap between $\V$ and the product $\W\H$.
NMFs have many applications in diverse fields and thus have drawn the attention of researchers and practitioners.
The problem is that the computation is intractable;
it was shown to be NP-hard by Vavasis \cite{Vav09}.

Arora et al.\ \cite{Aro12a} further investigated the complexity of the NMF problem.
They proposed to use an assumption, called separability, to remedy the issue.
The notion of separability was originally introduced by 
Donoho and Stodden in \cite{Don03} as a way of discussing the uniqueness of NMFs.
Arora et al.\ showed that,  if we place the separability assumption on the input matrix $\V$, 
then  the NMF problem turns out to be tractable; we can find the factors $\W$ and $\H$ without much effort such that $\V=\W\H$.
Let us say that a matrix is separable if it satisfies the separability assumption. 
The application range of separable NMFs is restricted in comparison to NMFs,
but they have still important applications, such as topic extraction from documents \cite{Aro12b, Aro13}
and unmixing of hyperspectral images \cite{Gil14b, Ma14}.
Other applications can be found in \cite{Fu19, Gil20}.
So far, several algorithms have been developed for solving separable NMF problems.
Separable matrices arising from applications should be perturbed by noise.
Hence, it is desirable that an algorithm is robust against noise; 
even if noise is added to a separable matrix, 
the algorithm should be able to find the factors whose product well approximates the noisy separable matrix.

Bittorf et al.\ \cite{Bit12} proposed an algorithm, referred to as Hottopixx, for separable NMF problems.
Their development is based on the observation that a certain feature of a separable matrix can be captured using linear programming (LP),
and an optimal solution of the LP serves as a guide for solving the separable NMF problem.
They showed that Hottopixx is robust to noise.
Their result needs somewhat strong assumption.
Roughly speaking, they assume that the columns of the separable matrix do not overlap.
The assumption is not reasonable when dealing with applications such as topic extraction from documents and unmixing of hyperspectral images.
Gillis \cite{Gil13} pointed out this issue and suggested a resolution.
He developed postprocessing for Hottopixx and showed that 
with it Hottopixx is robust to noise without the assumption Bittorf et al.\ \cite{Bit12} put.

There is a drawback with Hottopixx and its postprocessing.
They require three input data: a noisy separable matrix, the factorization rank, and the noise level.
In the applications we mentioned above, 
we often encounter the situation in which  the factorization rank can be estimated in advance.
Meanwhile, it is unlikely that the noise level can be estimated in advance;
we thus need to estimate it and its estimation is not an easy task.
For that reason, most of the algorithms for solving separable NMF problems,
such as VCA \cite{Nas05}, SPA \cite{Gil14b}, SNPA \cite{Gil14c} and ER \cite{Miz14},
are designed to receive two input data: a noisy separable matrix and the factorization rank.
Several drawbacks of Hottopixx are listed by Gillis and Luce in \cite{Gil14a} 
and the drawback we mentioned above is one of them.

The main contribution of this paper is to overcome the drawback.
We present a refinement of Hottopixx and its postprocessing
that takes a noisy separable matrix and the factorization rank as input, 
but does not need prior knowledge of the noise level.
We show that the refinement has almost the same robustness to noise as the original algorithm.
The results are summarized 
in Theorems \ref{Thm: Robustness of refined algo} and \ref{Thm: Robustness of refined algo with postprocessing} 
of Section \ref{Sec: Main results}.
In addition, we demonstrate in experiments the effectiveness of our refinement.

This paper is organized as follows.
In Section \ref{Sec: Problem and preliminaries}, we formulate the separable NMF problem 
and explain the assumption and parameters used in our analysis.
Section \ref{Sec: Main results} presents the main results and compares them with the results of previous studies.
Sections \ref{Sec: Refinement of Hottopixx} and \ref{Sec: Refinement of Hottopixx with postprocessing} 
describe the proposed algorithms and examine their robustness to noise;
the refinement of Hottopixx is in Section \ref{Sec: Refinement of Hottopixx} and 
the refinement of postprocessing is in Section \ref{Sec: Refinement of Hottopixx with postprocessing}.
Section \ref{Sec: Experiments} describes experiments.

\subsection{Notation and Symbols}
We write $\zero$ for a vector of all zeros, $\one$ for a vector of all ones, $\e_i$ for the $i$th unit vector, 
and $\I$ for the identity matrix.
The symbol $\zero$ is also used for a matrix of all zeros; 
in particular, $\zero_{m \times n}$ for an $\by{m}{n}$ matrix of all zeros.

The notation $\a(i)$ denotes the $i$th element of $\a \in \Real^n$.
Let $\A \in \Real^{m \times n}$.
The rows, columns and elements are denoted as follows:
$\A(i,:)$ for the $i$th row,  $\A(:,j)$ or $\a_j$ for the $j$th column, and $\A(i,j)$ for the $(i,j)$th element. 
Let $I \subset \{1, \ldots m\}$ and $J \subset \{1, \ldots, n\}$.
The notation $\A(I, :)$ denotes the submatrix obtained by eliminating  rows $\A(i,:)$ for all indices $i$ in the complement of $I$,
and $\A(:, J)$ that by eliminating columns $\A(:, j)$ for all indices $j$ in the complement of $J$.

The notation $\| \cdot \|_p$ denotes the $L_p$ norm of a vector or a matrix,
$\| \cdot \|_F $ the Frobenius norm of a matrix,
$\trace(\cdot)$ the trace of a square matrix, and 
$\diag(\cdot)$ a vector composed of diagonal elements of a square matrix.
i.e., $\diag(\B) = [\B(1,1), \ldots, \B(n,n)]^\trans$ for $\B \in \Real^{n \times n}$.
For positive integers $r$ and $n$,
the symbol $R$ denotes the set of consecutive integers from $1$ to $r$, and $N$ that from $1$ to $n$.
For $S \subset N$, we denote by $S^c$ the complement of $S$.
For $a,b \in \Real$ with $a < b$, the notation $(a,b)$ denotes the open interval $\{x \in \Real: a < x < b\}$, 
and $[a, b]$ the closed interval $\{x \in \Real: a \le x \le b\}$.

\section{Problem and Preliminaries} \label{Sec: Problem and preliminaries}

Let $\V \in \Real_+^{d \times n}$ have an exact NMF $\V = \W\H$ for $\W \in \Real_+^{d \times r}$ and $\H \in \Real_+^{r \times n}$.
\emph{Separability} assumes that it can be further written as 
\begin{align}
\V = \W \H \ \mbox{for} \ \W \in \Real_+^{d \times r} \ \mbox{and} \ \H = [\I, \bar{\H}]\Pib \in \Real_+^{r \times n}
 \label{Exp: SepNMF}
\end{align}
where $\I$ is an \by{r}{r} identity matrix, $\bar{\H}$ is an \by{r}{(n-r)} nonnegative matrix, and $\Pib$ is an \by{n}{n} permutation matrix.
When a nonnegative matrix is written in the form shown in (\ref{Exp: SepNMF}), 
we say that it is \emph{$r$-separable} or  simply \emph{separable}.
Separability means that all columns of $\W$ appear in those of $\V$;
that is, there is a map $\phi : R \rightarrow N$ such that $\w_j = \v_{\phi(j)}$ for each $j = 1, \ldots, r$.
We call the matrix $[\v_{\phi(1)}, \ldots, \v_{\phi(r)}]$, 
which is equivalent to $\W$, the \emph{basis} of $\V$: 
in particular, $\v_{\phi(j)}$ is the \emph{basis column} and $\phi(j)$ the \emph{basis index}.
We call $r$ the \emph{factorization rank} of $\V$.
We formulate the separable NMF problem as follows:
\begin{prob}
 Given a separable matrix $\V$ and factorization rank $r$, find the basis of $\V$. 
\end{prob}
Separable matrices arising from applications would contain noise.
\emph{Noisy separability} assumes that $\N \in \Real^{d \times n}$ is added to a separable matrix $\V \in \Real_+^{d \times n}$ 
such that 
\begin{align*}
 \A = \V + \N.
\end{align*}
We call $\N$ the \emph{noise} added to the separable matrix $\V$.
If a matrix is in the form above, we say that it is \emph{noisy separable}.
When dealing with applications, 
it is desirable that, even if a separable matrix contains noise,
an algorithm for solving separable NMF problems can still find a near-basis.
Given a noisy separable matrix $\A = \V + \N$, we say that the algorithm is {\it robust to noise} if 
it can find the column index set $J$ such that $\A(:,J)$ is close to the basis of $\V$.

Our analysis put the following assumption on a matrix $\A$.
\begin{asm} \label{Asm: Noisy separable matrix A}
 $\A = \V + \N \in \Real^{d \times n}$, 
 where $\V \in \Real_+^{d \times n}$ is $r$-separable of the form 
 $\V = \W\H = \W [\I, \bar{\H}]\Pib$ shown in (\ref{Exp: SepNMF}) 
 and $\N \in \Real^{d \times n}$ is noise.
 Moreover,
\begin{enumerate}[{\normalfont (a)}]
 \item every column of $\V, \W$ and $\H$ has unit $L_1$ norm, and 
 \item the noise $\N$ satisfies $\| \N \|_1 \le \epsilon$ for some real number $\epsilon$ satisfying $0 \le \epsilon < 1$.
\end{enumerate}
\end{asm}
We call $\epsilon$ the \emph{noise level} involved in $\A$.
As described in \cite{Aro12a, Gil12}, we can assume without loss of generality that part (a) holds.
Our analysis uses parameters $\kappa, \omega$ and $\beta$, 
which were introduced by Gillis \cite{Gil13} for the analysis of Hottopixx.
Let $\A = \V + \N \in \Real^{d \times n}$ where $\V \in \Real_+^{d \times n}$ is $r$-separable of the form
$\V = \W\H = \W [\I, \bar{\H}]\Pib$ shown in (\ref{Exp: SepNMF}) and $\N \in \Real^{d \times n}$ is noise.
The parameters $\kappa$ and $\omega$ are defined in terms of $\W$ by 
\begin{align*}
  \kappa  &= \min_{1 \le j \le r} \min_{\z \ge \zero}  \| \w_j - \W(:, R \setminus \{ j \}) \z \|_1, \\
  \omega  &= \min_{1 \le j_1 \neq j_2 \le r} \| \w_{j_1} - \w_{j_2} \|_1.
\end{align*}
They satisfy the relation
\begin{align} \label{Exp: kappa is less than omega}
 \kappa \le \omega.
\end{align}
It is easy to verify that it holds.
Let $j_1, j_2 \in R$ with $j_1 \neq j_2$ satisfy $\omega = \|\w_{j_1} - \w_{j_2} \|_1$.
Then, there exists an integer $\ell \in R$ such that $\W(:, R \setminus \{ j_1 \}) \e_\ell = \w_{j_2}$.
Hence, 
\begin{align*}
 \kappa \le \| \w_{j_1} - \W(:, R \setminus \{j_1\}) \e_\ell \|_1 = \| \w_{j_1} - \w_{j_2} \|_1 = \omega.
\end{align*}
Let Assumption~\ref{Asm: Noisy separable matrix A}(a) hold. Then, we can bound $\kappa$ and $\omega$ as 
\begin{align}
 & 0 \le \kappa \le 1, \label{Exp: Bounds on kappa} \\
 & 0 \le \omega \le 2. \label{Exp: Bounds on omega}
\end{align}
The lower bounds come from the definitions of $\kappa$ and $\omega$.
For the upper bounds, we find that 
$\kappa \le \| \w_j \|_1 = 1$ for any $j \in R$, and 
$\omega \le \| \w_{j_1} - \w_{j_2} \|_1 \le \| \w_{j_1} \|_1 + \| \w_{j_2} \|_1 = 2$ for any different $j_1, j_2 \in R$.
The parameter $\beta$ is defined in terms of the submatrix $\bar{\H}$ of $\H$ by
\begin{align*}
 \beta = \max_{1 \le i \le r, \ 1 \le j \le n-r} \bar{\H}(i,j).
\end{align*}
Let  Assumption~\ref{Asm: Noisy separable matrix A}(a) hold. Then, $\beta$ satisfies $0 \le \beta \le 1$.
In particular, if $\beta = 1$, there are columns of $\bar{\H}$ such that one element is $1$ and the others are $0$.
This means that there are duplicate basis columns.

\section{Main Results} \label{Sec: Main results}
Here, we present the main results in the form of 
Theorems \ref{Thm: Robustness of refined algo} and \ref{Thm: Robustness of refined algo with postprocessing}.
We refine Hottopixx of Bittorf et al.\ \cite{Bit12}.
Our refinement uses the optimization model $\OurModel$, which is shown in Section \ref{Subsec: Algorithm of refined Hottopixx}.
Algorithm \ref{Algo: Refined algo} of the section describes the details of the refinement.
Our first result, which states the robustness of Algorithm \ref{Algo: Refined algo} to noise, is as follows:

\begin{thm}
 \label{Thm: Robustness of refined algo}
 Let $\A$ satisfy Assumption \ref{Asm: Noisy separable matrix A}. Assume $\kappa > 0$.
 Run the refinement of Hottopixx, i.e., Algorithm \ref{Algo: Refined algo}, on the input $(\A, r)$.
 If 
 \begin{align*}
  \epsilon \le \frac{\kappa (1 - \beta)}{9 (r+1)},
 \end{align*}
 then, after suitably rearranging the columns of $\W$, 
 the output $\W_\out$ satisfies $\|\W - \W_\out \|_1 \le \epsilon$.
\end{thm}

If the noise level $\epsilon$ is positive and the basis columns overlap, i.e., $\beta = 1$,
the theorem is invalid and does not say anything about the robustness of Algorithm \ref{Algo: Refined algo} to noise.
To cope with this issue, 
we develop postprocessing that ensures the algorithm's robustness to noise even in such a case.
Postprocessing for that purpose was proposed by Gillis \cite{Gil13}, and here, we refine it.
A detailed description of the refinement is given 
in Algorithm \ref{Algo: Refined algo with postprocessing} of Section \ref{Subsec: Algorithm of postprocessing}.
Our second result, which states the robustness of Algorithm \ref{Algo: Refined algo with postprocessing} to noise, is as follows:

\begin{thm}
 \label{Thm: Robustness of refined algo with postprocessing}
 Let $\A$ satisfy Assumption \ref{Asm: Noisy separable matrix A}.
 Run the refinement of Hottopixx with postprocessing, i.e., Algorithm \ref{Algo: Refined algo with postprocessing}, 
 on the input $(\A, r)$. If 
 \begin{align*}
  \epsilon < \frac{\kappa \omega}{578(r+1)},
 \end{align*}
 then, after suitably rearranging the columns of $\W$, the output $\W_\out$ satisfies 
 \begin{align*}
  \| \W - \W_\out \|_1 \le \frac{136(r+1)}{\kappa} \epsilon.
 \end{align*}
 In particular, if 
 \begin{align*}
  \epsilon <  \frac{\kappa^2}{289(r+1)^2},
 \end{align*}
 then, after suitably rearranging the columns of $\W$, the output $\W_\out$ satisfies 
 \begin{align*}
  \| \W - \W_\out  \|_1 \le 8\sqrt{\epsilon}.
 \end{align*}
\end{thm}

Theorem \ref{Thm: Robustness of refined algo with postprocessing} tells us that 
there is a range of noise intensity that Algorithm \ref{Algo: Refined algo with postprocessing} is robust to 
even if there are duplicate basis columns.
From the relation $\kappa \le \omega$ shown in (\ref{Exp: kappa is less than omega}),
we can see that $\frac{\kappa^2}{289(r+1)^2} \le \frac{\kappa \omega}{578(r+1)}$ holds, and
$\frac{\kappa^2}{289(r+1)^2}$ is $\frac{2}{r+1}$ times smaller than $\frac{\kappa \omega}{578(r+1)}$.
The theorem tells us that, 
if $\epsilon$ satisfies $\epsilon \le \frac{\kappa \omega}{578(r+1)}$, 
the error of the output $\W_\out$ relative to the basis $\W$ can be bounded by using $\epsilon, r, \kappa$;
in particular, if $\epsilon$ is small and satisfies $\epsilon \le \frac{\kappa^2}{289(r+1)^2}$,
the error bound depends on only $\epsilon$.
Note that it remains an open question how tight 
the bounds shown in Theorems \ref{Thm: Robustness of refined algo} and \ref{Thm: Robustness of refined algo with postprocessing} are.
This is a topic for further research.

Now, let us review the previous work on Hottopixx and compare our results with the previous ones.
Arora et al.\ \cite{Aro12a} proposed the first algorithm with provable guarantees for solving separable NMF problems.
Motivated by that work, Bittorf et al.\ \cite{Bit12} developed Hottopixx.
Let $\A = \V + \N$ 
where $\V \in \Real_+^{d \times n}$ is $r$-separable of the form $\V = \W\H$ shown in (\ref{Exp: SepNMF}) 
and $\N \in \Real^{d \times n}$ is noise satisfying $\| \N \|_1 \le \epsilon$ for some nonnegative real number $\epsilon$.
Hottopixx is based on the optimization model $\BRRT$ and require $(\A, r, \epsilon)$ as its input.
The details of the algorithm and  $\BRRT$ are given in Section \ref{Subsec: Algorithm of refined Hottopixx}.
Bittorf et al.\ showed that Hottopixx is robust to noise. 
However, it was unclear whether one can ensure its robustness in the case that there are duplicate basis columns.

Gillis \cite{Gil13} and Gillis and Luce \cite{Gil14a} pursued a line of research 
that examined the robustness of Hottopixx.
Tables \ref{Tab: algorithms without postprocessing} and \ref{Tab: algorithms with postprocessing}
summarize their results as well as ours.
The first column lists  the input data of the algorithms;
the second one lists the optimization model whose details are given in Section \ref{Subsec: Algorithm of refined Hottopixx};
the third one lists the assumptions imposed on the analysis; 
and the fourth and fifth ones list the robustness results obtained by the analysis,
i.e., the bound on the noise level and the error of the output relative to the basis.

Gillis \cite{Gil13} investigated the robustness of Hottopixx.
He started by analyzing the case where there are no duplicate basis columns.
The analysis suggested that the use of postprocessing makes it possible to enhance its robustness.
He then developed postprocessing and showed that Hottopixx with the postprocessing is robust to noise 
even when there are duplicate basis columns.
The results on Hottopixx (Theorem 2.3 of \cite{Gil13}) are summarized in the second row of Table \ref{Tab: algorithms without postprocessing},
and those on Hottopixx with the postprocessing (Theorem 3.5 of \cite{Gil13}) are in the second row of Table \ref{Tab: algorithms with postprocessing}.

Gillis and Luce \cite{Gil14a} developed a refinement of Hottopixx. 
Their refinement is based on the optimization model $\GL$, and it requires $(\A, \epsilon)$ as input.
The details of the algorithm and $\GL$ are given in Section \ref{Subsec: Algorithm of refined Hottopixx}.
They showed that the refinement is robust to noise.
The results (Theorem 2 of \cite{Gil14a}) are summarized in the third row of Table \ref{Tab: algorithms without postprocessing}.
Here, $\rho$ is a parameter that is set to a positive real number.
The advantage of the refinement over Hottopixx is that it does not require prior knowledge of the factorization rank $r$ 
of the matrix $\A$ we want to factorize, and the robustness result does not depend on $r$.
They also incorporated the postprocessing of Gillis \cite{Gil13} into the refinement, 
and showed that the same result as Theorem 3.5 of \cite{Gil13} 
holds for the refinement with the postprocessing.
The results (Theorem 7 of \cite{Gil14a}) are summarized in the third row of Table \ref{Tab: algorithms with postprocessing}.

\begin{table}[t]
 \centering
 \caption{Comparison of our result (Theorem \ref{Thm: Robustness of refined algo}) 
 with those of  Gillis (Theorem 2.3 of \cite{Gil13}) and  Gillis and Luce (Theorem 2 of \cite{Gil14a}) 
 for algorithms without postprocessing.
 The algorithm of Gillis and Luce  uses a parameter $\rho$ that is set to a positive real number.}
 \label{Tab: algorithms without postprocessing}
\begin{tabular}{p{25mm}|p{15mm}p{15mm}p{40mm}p{25mm}p{20mm}}
 \toprule
 & Input             & Model       & Assumption & Noise level      & Error \\
 \midrule
 Our result 
 & $\A, r$           & $\OurModel$ & Assumption \ref{Asm: Noisy separable matrix A}, $\kappa > 0$  
 & $\frac{\kappa (1 - \beta)}{9 (r+1)}$ & $\epsilon$ \\
 Gillis 
 & $\A, r, \epsilon$ &  $\BRRT$    & Assumption \ref{Asm: Noisy separable matrix A}, $\kappa > 0$   
 & $\frac{\kappa (1 - \beta)}{9 (r+1)}$ & $\epsilon$ \\
 Gillis and Luce 
 & $\A, \epsilon$    &   $\GL$     & Assumption \ref{Asm: Noisy separable matrix A}, $\kappa > 0$   
 & $\frac{\kappa (1 - \beta) \min\{1, \rho\}}{5 (\rho+2)}$ & $\epsilon$ \\
 \bottomrule
\end{tabular}

 \caption{Comparison of our result (Theorem \ref{Thm: Robustness of refined algo with postprocessing}) 
 with those of Gillis (Theorem 3.5 of \cite{Gil13}) and  Gillis and Luce (Theorem 7 of \cite{Gil14a}) for algorithms with postprocessing.}
 \label{Tab: algorithms with postprocessing}
\begin{tabular}{p{25mm}|p{15mm}p{15mm}p{40mm}p{25mm}p{20mm}}
 \toprule
 & Input     & Model      & Assumption & Noise level    & Error \\
 \midrule
 Our result & $\A, r$  & $\OurModel$ & Assumption \ref{Asm: Noisy separable matrix A}  
	    &  $\frac{\kappa \omega}{578(r+1)}$ & $\frac{136(r+1)}{\kappa} \epsilon$ \\

 Gillis & $\A, r, \epsilon$ & $\BRRT$ & Assumption \ref{Asm: Noisy separable matrix A} 
        & $\frac{\kappa \omega}{99(r+1)}$ & $\frac{49(r+1)}{\kappa} \epsilon + 2 \epsilon$ \\

 Gillis and Luce & $\A, r, \epsilon$ & $\GL$   & Assumption \ref{Asm: Noisy separable matrix A}  
                 & $\frac{\kappa \omega}{99(r+1)}$ & $\frac{49(r+1)}{\kappa} \epsilon + 2 \epsilon$ \\
 \bottomrule
\end{tabular}
\end{table}

Let us compare our results with those of Gillis \cite{Gil13} and Gillis and Luce \cite{Gil14a}.
We can see from Tables \ref{Tab: algorithms without postprocessing} and \ref{Tab: algorithms with postprocessing} that 
Algorithm \ref{Algo: Refined algo} is as robust as Hottopixx,
and Algorithm \ref{Algo: Refined algo with postprocessing} is almost as robust as  Hottopixx and the refinement of Gillis and Luce 
with the postprocessing of Gillis.
The assumptions of our analysis are the same as theirs.
There is a difference in the input data: 
$(\A,r)$ for our algorithms and $(\A, r, \epsilon)$ or $(\A, \epsilon)$ for the existing algorithms.
We often encounter a situation in which the factorization rank $r$ is available in advance 
in applications such as topic extraction from documents and unmixing of hyperspectral images.
Hence, it is reasonable to assume that a noisy separable matrix $\A$ and the factorization rank $r$ will be given as input.
As mentioned in Section \ref{Sec: Introduction}, most of the algorithms for solving separable NMF problems 
are designed to take $(\A,r)$ as input.
The advantage of our algorithms over the existing ones 
is they run on $(\A,r)$ that does not include prior knowledge of the noise level $\epsilon$ 
and yet have almost the same robustness to noise as the existing ones.

\section{Refinement of Hottopixx} \label{Sec: Refinement of Hottopixx}

\subsection{Algorithm} \label{Subsec: Algorithm of refined Hottopixx}
Our refinement of Hottopixx is described in Algorithm \ref{Algo: Refined algo}. 
\begin{algorithm}[h]
 \caption{Refinement of Hottopixx}
 \label{Algo: Refined algo}
 \smallskip
 Input: $\A \in \Real^{d \times n}$ and a positive integer $r$. \\
 Output: $\W_{\out} \in \Real^{d \times r}$.

 \begin{enumerate}[1.]
  \item If there are duplicate columns in $\A$, keep one of them and remove all the rest.

  \item Compute the optimal solution $\OptSol$ of the problem $\OurModel(\A,r)$.
	Set $\p = \diag(\OptSol)$.
  \item Let $\W_{\out} = \A(:,J)$ for the index set $J$ corresponding to the $r$ largest elements of $\p$, 
	and return $\W_{\out}$.
 \end{enumerate}
\end{algorithm}
For the input $\A$ and $r$,
step 2 constructs and solves the optimization problem with variable $\X \in \Real^{n \times n}$, 
\begin{alignat*}{4}
 & \OurModel(\A, r): & \quad  & \text{Minimize}   & \quad & \|\A - \A \X \|_1    & \quad & \\
 &                   &        & \text{subject to} &       & \trace(\X) = r,      &       & \\
 &                   &        &                   &       & \X(i,i) \le 1        &       & \text{for all} \ i \in N, \\
 &                   &        &                   &       & \X(i,j) \le \X(i,i)  &       & \text{for all} \ i,j \in N, \\
 &                   &        &                   &       & \X(i,j) \ge 0        &       & \text{for all} \ i,j \in N.
\end{alignat*}
Throughout this paper, we  use $\OptSol$ to denote the optimal solution and  
$\OptVal$ to denote the optimal value $\|\A - \A \OptSol \|_1$.
By introducing new variables $\Y \in \Real^{d \times n}$ and $z \in \Real$, 
the problem above can be reduced to an LP problem, 
since the minimization of $\|\A - \A \X \|_1$ is equivalent to the minimization of $z$ under the constraints:
$-\Y \le \A - \A\X \le \Y$ and $\sum_{i = 1}^{d} \Y(i,j) \le z$ for all $j \in N$.
We use $\OurModel'$ to denote the LP problem.
It should be noted that $\OurModel'$ has $n^2 + dn + 1$ variables and $2n^2 + 2dn + n + 1$ constraints.
Hence, the size of $\OurModel'$ may be rather large.
Step 1 performs the preprocessing on the input matrix.
Although Hottopixx does not contain this step, 
Algorithm \ref{Algo: Refined algo} must have it.
See Remark \ref{Remark: theorem does not hold} at the end of this section for the reason.

Here, let us recall Hottopixx of Bittorf et al.\ \cite{Bit12} and the refinement of Gillis and Luce \cite{Gil14a}.
Bittorf et al.\ looked at a certain feature of separable matrices
and developed Hottopixx on the basis of that observation.
Let $\A$ satisfy Assumption \ref{Asm: Noisy separable matrix A}.
Then, it can be written as $\A = \V + \N$,  where $\V \in \Real_+^{d \times n}$ is $r$-separable of the form
$\V = \W\H = \W [\I, \bar{\H}]\Pib$ shown in (\ref{Exp: SepNMF}) 
and $\N \in \Real^{d \times n}$ is noise.
Using an \by{n}{n} permutation matrix $\Pib$ and the \by{r}{(d-r)} nonnegative matrix $\bar{\H}$,
we construct the  matrix
\begin{align} \label{Exp: X0}
 \X_0 = \Pib^{-1}
 \left[
 \begin{array}{c|c}
  \I    & \bar{\H} \\
  \hline
  \zero_{\by{(n-r)}{r}} & \zero_{\by{(n-r)}{(n-r)}}
 \end{array}
 \right]
  \Pib \in \Real^{n \times n}
\end{align}
where $\I$ is an identity matrix of size $r$.
We make the following observations:
\begin{itemize}
 \item The basis of $\V$ can be identified by using $\X_0$, since 
       the diagonal entries of $\X_0$  
       are $0$ or $1$ and the positions with $1$ correspond to the basis indices of $\V$.
 \item $\X_0$ satisfies 
       \begin{align} \label{Exp: AX0 approximates A}
	\| \A - \A\X_0 \|_1 \le 2 \epsilon.
       \end{align}
\end{itemize}
The second observation comes from the fact that we have
\begin{align*}
 \V\X_0 = \W [\I, \bar{\H}] \Pib 
 \Pib^{-1}
 \left[
 \begin{array}{c|c}
  \I    & \bar{\H} \\
  \hline
   \zero & \zero
 \end{array}
 \right]
 \Pib
 = \W[\I, \bar{\H}] \Pib = \V,
\end{align*}
which gives
 \begin{align*}
  \|\A - \A\X_0 \|_1 = \|\V + \N  - (\V + \N) \X_0\|_1 & = \|\N -  \N \X_0\|_1  
  & \\
  & \le \|\N\|_1 + \|\N\|_1 \|\X_0\|_1  
  & \\
  & \le 2 \epsilon  
  & \text{(by Assumption \ref{Asm: Noisy separable matrix A})}.
 \end{align*}
To compute $\X_0$ approximately,
Bittorf et al.\ proposed to solve an optimization problem with variable $\X \in \Real^{n \times n}$, 
\begin{alignat*}{4}
 & \BRRT(\A, r, \epsilon):      & \quad & \mbox{Minimize}   & \quad & \f^\trans \diag(\X)                & \quad &  \\
 &                              &       & \mbox{subject to} &       & \| \A - \A\X \|_1 \le 2 \epsilon,  &       &  \\ 
 &                              &       &                   &       & \trace(\X) = r,                    &       &  \\
 &                              &       &                   &       & \X(i,i) \le 1                      &       & \mbox{for all} \ i \in N, \\
 &                              &       &                   &       & \X(i,j) \le \X(i,i)                &       & \mbox{for all} \ i,j \in N,  \\
 &                              &       &                   &       & \X(i,j) \ge 0                      &       & \mbox{for all} \ i,j \in N.
\end{alignat*}
Here, $\f$ is a parameter set by the user: it can be chosen to be any $n$-dimensional vector with distinct elements.
The problem $\BRRT$ can be reduced to an LP.
Hottopixx is the same as performing steps 2 and 3 of Algorithm \ref{Algo: Refined algo} with a replacement of $\OurModel(\A,r)$ in step 2
by $\BRRT(\A,r,\epsilon)$.
It thus requires $(\A, r, \epsilon)$ as input.

Gillis and Luce \cite{Gil14a} refined Hottopixx. 
They proposed to solve an optimization problem with variable $\X \in \Real^{n \times n}$, 
\begin{alignat*}{4}
 & \GL(\A, \epsilon):  & \quad & \mbox{Minimize}   & \quad & \g^\trans \diag(\X)                  & \quad &  \\
 &                     &       & \mbox{subject to} &       & \| \A - \A\X \|_1 \le \rho \epsilon,  &       &  \\ 
 &                     &       &                   &       & \X(i,i) \le 1                        &       & \mbox{for all} \ i \in N, \\
 &                     &       &                   &       & \X(i,j) \le \X(i,i)                  &       & \mbox{for all} \ i,j \in N,  \\
 &                     &       &                   &       & \X(i,j) \ge 0                        &       & \mbox{for all} \ i,j \in N.
\end{alignat*}
Here, $\g$ and $\rho$ are parameters set by the user: 
$\g$ can be chosen to be any $n$-dimensional vector with distinct positive elements 
and $\rho$ a positive value.
As in the case of $\BRRT$, the problem $\GL$ can be reduced to an LP.
Their algorithm computes the optimal solution of $\GL$ and constructs an index set corresponding to 
diagonal entries larger than $1 - \frac{\min\{1, \rho\}}{2}$.
Hence, it takes as input $(\A, \epsilon)$ and does not require $r$ as input.

\begin{remark}\label{Remark: theorem does not hold}
 If Algorithm \ref{Algo: Refined algo} does not contain step 1,
 it may fail to find a basis from separable matrices with duplicate basis columns.
 For instance, consider 
 \begin{align*}
  \V = 
  \left[
  \begin{array}{rrrrr}
   1 & 0 & 1 & 1 & 0 \\
   0 & 1 & 0 & 0 & 1
  \end{array}
  \right].
 \end{align*}
 This is $2$-separable with $\beta = 1$ since it can be written as $\V = \W\H$ by letting $\W = \I$ and $\H = \V$.
 Suppose that the algorithm receives $(\A, r)$ by letting $\A = \V$ and $r = 2$ as input.
 Consider the two matrices, 
 \begin{align*}
  \X_1 = 
  \left[
  \begin{array}{rrrrr}
   1 & 0 & 1 & 1 & 0 \\
   0 & 1 & 0 & 0 & 1 \\
   0 & 0 & 0 & 0 & 0 \\
   0 & 0 & 0 & 0 & 0
  \end{array}
  \right]
  \quad \mbox{and} \quad
  \X_2 = 
  \left[
  \begin{array}{rrrrr}
   1/3 & 0   & 1/3 & 1/3 & 0 \\
   0   & 1/2 & 0   & 0   & 1/2 \\
   1/3 & 0   & 1/3 & 1/3 & 0 \\
   1/3 & 0   & 1/3 & 1/3 & 0 \\
   0   & 1/2 & 0   & 0   & 1/2 
  \end{array}
  \right].
 \end{align*}
 Both $\X_1$ and $\X_2$ are optimal solutions of problem $\OurModel(\A,r)$,
 since they satisfy all the constraints and  $\| \A - \A\X_1 \|_1 = \| \A - \A\X_2 \|_1 = 0$.
 If Algorithm \ref{Algo: Refined algo} skips step 1 
 and finds $\X_2$ in step 2, then it constructs  $J = \{2,5\}$ in step 3.
 We have $\W_{\out} \neq \W $, since $\W_{\out} = \A(:,J)= \V(:,J)$ and $\W = \I$.
\end{remark}

\subsection{Analysis} \label{Subsec: Analysis on refined Hottopixx}

The optimal value $\theta$ of problem $\OurModel$ is related to the noise level $\epsilon$ 
involved in separable matrices.
Actually, from the observation Bittorf et al.\ made in \cite{Bit12}, 
we can easily see that $\theta \le 2\epsilon$ holds.

\begin{lem}\label{Lem: Relation of epsilon and optimal value}
 Let $\A$ satisfy Assumption \ref{Asm: Noisy separable matrix A}.
 Then, the optimal value $\OptVal$ of problem $\OurModel(\A, r)$ satisfies $\OptVal \le 2 \epsilon$.
\end{lem}
\begin{proof}
 Since $\A$ satisfies Assumption \ref{Asm: Noisy separable matrix A},
 it is given by $\A = \V + \N$  where $\V \in \Real_+^{d \times n}$ is $r$-separable of the form
 $\V = \W\H = \W [\I, \bar{\H}]\Pib$ shown in (\ref{Exp: SepNMF}) 
 and $\N \in \Real^{d \times n}$ is noise.
 Using the permutation matrix $\Pib$ and the nonnegative matrix $\bar{\H}$,  
 we construct the  matrix $\X_0$ that is shown in (\ref{Exp: X0}), i.e.,
\begin{align*}
 \X_0 = \Pib^{-1}
 \left[
 \begin{array}{c|c}
  \I    & \bar{\H} \\
  \hline
  \zero & \zero
 \end{array}
 \right]
  \Pib \in \Real^{n \times n}.
\end{align*}
 Since Assumption~\ref{Asm: Noisy separable matrix A}(a) holds,
 we can check that $\X_0$ is a feasible solution of  $\OurModel(\A, r)$. Hence,
 the objective function value at $\X_0$ satisfies
 $\OptVal \le \|\A - \A\X_0 \|_1$.
 In addition, as shown in (\ref{Exp: AX0 approximates A}), we have $\|\A - \A\X_0 \|_1 \le 2 \epsilon$.
 Consequently,  $\theta \le 2 \epsilon$ holds.
\end{proof}

Let $\A$ satisfy Assumption \ref{Asm: Noisy separable matrix A}.
Let $I$ be a set of basis indices of $\V$.
Gillis showed in Lemma 2.1 of \cite{Gil13} that a feasible solution $\X$ of problem $\BRRT$ has the following properties:
the $L_1$ norm of each column of $\X$ is less than about $1$, and $\V\X$ serves as a good approximation to $\V$.
Using the results, Gillis showed in Lemma 2.2 of \cite{Gil13} that 
the diagonal elements of $\X$ indexed by $I$ take higher values than the others.
Hence, we can construct $I$ by checking the values of the diagonal elements of $\X$.
Lemma \ref{Lem: Relation of epsilon and optimal value} implies that 
the optimal solution of problem $\OurModel$ is feasible for problem $\BRRT$.
Hence, the same results as in Lemmas 2.1 and 2.2 of \cite{Gil13} hold for the optimal solution of problem $\OurModel$.
Here, we formally describe these results as Lemmas \ref{Lem: Properties of optimal solution} and \ref{Lem: Score of basis}.

\begin{lem}\label{Lem: Properties of optimal solution}
 Let $\A$ satisfy Assumption \ref{Asm: Noisy separable matrix A}.
 Then, the optimal solution $\OptSol \in \Real^{n \times n}$ of problem $\OurModel(\A, r)$ satisfies
 \begin{align*}
  \| \OptSol(:,i) \|_1 \le 1 + \frac{4 \epsilon}{1 - \epsilon} 
  \quad \mbox{and} \quad 
  \| \v_i - \V\OptSol(:,i) \|_1 \le \frac{4 \epsilon}{1 - \epsilon}
 \end{align*}
 for $ i \in N$.
\end{lem}
The proof is almost the same as the one  of Lemma 2.1 in \cite{Gil13}.
We have included it in Appendix A to make the discussion self-contained.

\begin{lem} \label{Lem: Score of basis}
 Let $\A$ satisfy Assumption \ref{Asm: Noisy separable matrix A}.
 Assume $\kappa > 0$ and $\beta < 1$.
 Let $I$ be a set of basis indices of $\V$. 
 Let $\p = \diag(\OptSol)$ for the optimal solution $\OptSol$ of problem $\OurModel(\A, r)$.
 Then, the elements of $\p$ indexed by $I$ satisfy 
 \begin{align*}
  \p(i) \ge 1 - \frac{8 \epsilon}{\kappa (1 - \beta) (1 - \epsilon)}
 \end{align*}
 for every $i \in I$.
\end{lem}
We have included the proof in Appendix B. 
Our proof follows the one of Lemma 2.2 in \cite{Gil13}, although 
additional considerations are made; see Remark \ref{Remark: Observation}.
The key idea of the proof is as follows.
Since $\A$ satisfies Assumption \ref{Asm: Noisy separable matrix A},
it can be written as $\A = \V + \N$
where $\V$ is $r$-separable of the form $\V = \W\H = \W[\I, \bar{\H}]\Pib$ shown in (\ref{Exp: SepNMF}) 
and there is a map $\phi : R \rightarrow N$ such that $\w_j = \v_{\phi(j)}$ for each $j \in R$.
For $j \in R$ and $i = \phi(j) \in N$, let 
\begin{align*}
 \eta = \H(j,:)\OptSol(:,i). 
\end{align*}
We can see that $\eta$ is rewritten by using $\OptSol(i,i)$, which is equivalent to $\p(i)$,
due to $\H(j,i) = 1$, and evaluate the lower and upper bounds on $\eta$.
The result of the lemma follows from the bounds.

Now, we can prove Theorem \ref{Thm: Robustness of refined algo}. 
It follows from Lemma \ref{Lem: Score of basis}.
\begin{proof}[(Proof of Theorem \ref{Thm: Robustness of refined algo})]
 Let us consider the case of $\beta = 1$. 
 Here, we only have to show that,
 if $\A$ is separable with an overlap of basis columns,
 then the algorithm finds a set of  basis indices. 
 Separability means that duplicate basis columns appear in the columns of $\A$.
 Hence, after conducting step 1, the resulting matrix is separable with no overlapping basis columns.
 This reduces to the case of $\beta < 1$.
 
 Let us  move on to the case of $\beta < 1$.
 Step 2 solves problem $\OurModel(\A, r)$ and sets $\p = \diag(\OptSol)$ for the optimal solution $\OptSol$.
 Let $I$ be a set of basis indices of $\V$.
 Lemma \ref{Lem: Score of basis} tells us that 
 \begin{align*}
  \p(i) \ge 1 - \underbrace{\frac{8 \epsilon}{\kappa (1 - \beta) (1 - \epsilon)}}_{\mathrm{(A)}}
 \end{align*}
 holds for every $i \in I$.
 Since 
 \begin{align*}
  \epsilon \le \frac{\kappa (1 - \beta)}{9 (r+1)} \le \frac{1}{18},
 \end{align*}
 we have $1 - \epsilon \ge 17/18 > 8/9$.
 In light of this, the term $\mathrm{(A)}$ is bounded as follows:
 \begin{align*}
  \mathrm{(A)} < \frac{9 \epsilon}{\kappa (1 - \beta)} \le \frac{1}{r+1}.
 \end{align*}
 We thus obtain
 \begin{align} \label{Exp: High scores}
  \p(i) > \frac{r}{r+1} \quad \mbox{for} \ i\in I.
 \end{align}
 The first constraint of problem $\OurModel(\A, r)$ requires $\OptSol$ to satisfy
 $\trace(\OptSol) = r \equivSym \sum_{i \in N} \p(i) = r$.
 Hence, 
\begin{align*}
 r = \sum_{ i \in N} \p(i) = \sum_{ i \in I} \p(i) + \sum_{ i \in N \setminus I} \p(i).  
\end{align*}
 Combining it with inequality (\ref{Exp: High scores}) gives 
 \begin{align} \label{Exp: Low scores}
  \p(i) \le \frac{r}{r+1} \quad \mbox{for} \ i \in N \setminus I.
 \end{align}
 Since $I$ has $r$ elements, 
 inequalities (\ref{Exp: High scores}) and (\ref{Exp: Low scores}) ensure that
 the index set corresponding to the $r$ largest elements of $\p$ coincides with $I$.
 Hence, the index set $J$ constructed in step 3 coincides with $I$, which is the set of basis indices of $\V$.
 Consequently, after suitably rearranging the columns of $\W$, 
 the output $\W_\out = \A(:,J)$ satisfies $\| \W - \W_\out \|_1 \le \epsilon$.
\end{proof}

\section{Refinement of Hottopixx with Postprocessing} \label{Sec: Refinement of Hottopixx with postprocessing}

\subsection{Algorithm} \label{Subsec: Algorithm of postprocessing}
We explore the case where there are duplicate basis columns in the input matrix of Algorithm \ref{Algo: Refined algo}.
As shown in Section~\ref{Subsec: Analysis on refined Hottopixx},
the algorithm's guarantee of robustness to noise is founded upon Lemma \ref{Lem: Score of basis}.
However, the lemma does not hold any more, because $\beta = 1$ in this case.
To address this issue, we develop and incorporate postprocessing in the algorithm.

Let us outline our postprocessing first and give the details at the end of this section.
In what follows, we will assume that we are given $\A$ satisfying Assumption \ref{Asm: Noisy separable matrix A}.
We use the term \emph{cluster} to refer to the set of column indices of $\A$.
Although Lemma \ref{Lem: Score of basis} does not hold in the case where there are duplicate basis columns,
the optimal solution of problem $\OurModel$ still provides us with clues to finding clusters from which we can obtain near-basis columns.
For a cluster $S \subset N$ and $\p \in \Real_+^n$, 
define the \emph{score} of cluster $S$ by
\begin{align*}
 \SCORE(S, \p) = \sum_{u \in S} \p(u),
\end{align*}
and we call $\p$ a \emph{point list}.

Let $\mu > 0$ be a parameter and define
\begin{align} \label{Exp: anchor}
 T_j = \{u \in N : \|\a_u - \w_j \|_1 \le 2 \mu\}
\end{align}
for each $j \in R$. Here, $\a_u$ is the $u$th column of $\A$ and $\w_j$ is the $j$th column of $\W$.
We call $T_1, \ldots, T_r$ \emph{anchors} with  parameter $\mu$.
Let $\p = \diag(\OptSol)$ for the optimal solution $\OptSol$ of problem $\OurModel$, 
and choose the parameter $\mu$ of $T_j$, depending on the noise level $\epsilon$ involved in $\A$.
We show in Corollary \ref{Cor: Score of anchors if amount of noise is small} 
that the anchors $T_1, \ldots, T_r$ have high scores, i.e.,
\begin{align*}
 \SCORE(T_j, \p) > \frac{r}{r+1}
\end{align*}
for every $j \in R$.
Lemma 3.3 of \cite{Gil13} by Gillis implies that the same result holds 
for a feasible solution of problem $\BRRT$.

If we find all the anchors, then near-basis columns can be obtained by choosing one element from each anchor.
However, even if we use the point list obtained from the optimal solution of $\OurModel$, 
it is not an easy task to find anchors exactly.
We thus construct a collection $\FC$ of clusters that contains all the anchors, and observe the structure of $\FC$.
Our postprocessing algorithm is designed on the basis of this observation.
To describe $\FC$, we introduce $\Omega$, which is a collection of clusters, that will serve as the foundation of $\FC$.
Sort columns $\a_1, \ldots, \a_n$ of $\A$ by their $L_1$ distance to $\a_i$ in ascending order so that 
\begin{align*}
 \| \a_i - \a_{u_1} \|_1 \le \| \a_i - \a_{u_2} \|_1 \le \cdots \le \| \a_i - \a_{u_{n-1}} \|_1
\end{align*}
where $\{i, u_1, \ldots, u_{n-1}\} = N$. Then, construct 
\begin{align*}
 \Omega_i = \{\{i\}, \{i, u_1\}, \{i, u_1, u_2\}, \ldots, \{i, u_1, u_2, \ldots, u_{n-1}\}\}
\end{align*}
and let 
\begin{align*}
 \Omega = \bigcup_{i \in N} \Omega_i.
\end{align*}
For a cluster $S \in \Omega_i$, define the diameter of $S$ in $\Omega_i$ by
\begin{align*}
  \diam(S) = \max_{u \in S} \|\a_i - \a_u \|_1.
\end{align*}
For a point list $\p \in \Real_+^n$ and parameter $\mu$ used for constructing the anchors $T_1, \ldots, T_r$,
let 
\begin{align} \label{Exp: Fi}
 \FC_i(\p) = \left\{ S \in \Omega_i : \diam(S) \le 3 \mu, \ \SCORE(S, \p) > \frac{r}{r+1} \right\}
\end{align}
and 
\begin{align*}
 \FC(\p) = \bigcup_{i \in N} \FC_i(\p).
\end{align*}
In particular, if $\p$ is set as $\p = \diag(\OptSol)$ for the optimal solution $\OptSol$ of problem $\OurModel$,
we use the abbreviation $\FC_i$ for $\FC_i(\p)$ and the abbreviation $\FC$ for $\FC(\p)$. That is,
for $\p = \diag(\OptSol)$, 
\begin{align*} 
 \FC_i = \FC_i(\p) \quad \mbox{and} \quad \FC = \FC(\p).
\end{align*}

As mentioned above,
we have to choose $\mu$ depending on the noise level $\epsilon$ to ensure that the anchors can have high scores.
Hence, it is impossible to construct $\FC(\p)$.
But, it is possible to compute some of the clusters in $\FC(\p)$.
Consider a collection $\GC_i(\p)$ of clusters obtained by removing the condition $\diam(S) \le 3 \mu$ in $\FC_i(\p)$:
\begin{align} \label{Exp: Gi}
 \GC_i(\p) = \left\{ S \in \Omega_i :  \SCORE(S, \p) > \frac{r}{r+1} \right\}.
\end{align}
Let 
\begin{align*}
 \GC(\p) = \bigcup_{i \in N} \GC_i(\p).
\end{align*}
Unlike $\FC(\p)$, we can construct $\GC(\p)$.
Let $\hat{S} = \arg \min_{S \in \GC(\p)} \diam(S)$.
Since $\FC(\p) \subset \GC(\p)$, we have 
\begin{align*}
 \diam(\hat{S}) = \min_{S \in \GC(\p)} \diam(S) \le \min_{S \in \FC(\p)} \diam(S) \le 3 \mu.
\end{align*}
Hence, $\hat{S}$ belongs to $\FC(\p)$.
We can get it through $\GC(\p)$.

\begin{figure}[t]
 \centering
 \includegraphics[width=0.95\linewidth]{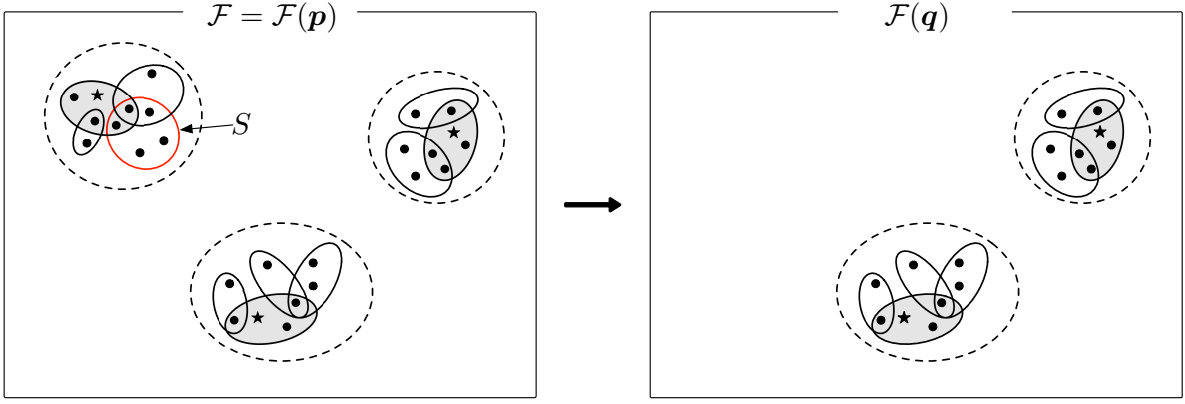}
 \caption{Illustration of $\FC = \FC(\p)$ and $\FC(\q)$ where
 $\p$ is a point list obtained from the optimal solution of problem $\OurModel$, 
 and $\q$ is a point list obtained by updating $\p$ such that
 $\q(u) = 0$ if $u$ belongs to the red-colored cluster $S$; otherwise, $\q(u) = \p(u)$:
 clusters (set of points surrounded by an oval), 
 anchors (set of points surrounded by an oval filled with gray color),
 the components $\bar{\FC}_i$ of $\FC$ (collection of clusters surrounded by a dotted oval), and basis columns (star).}
 \label{Fig: Structure of F}
\end{figure}

Let us look at $\FC$, which is an abbreviation of $\FC(\p)$ 
with the point list $\p$ obtained from the optimal solution of problem $\OurModel$.
We show in Lemma \ref{Lem: Cluster in F intersects with some of anchors} 
that any cluster in $\FC$ always has a common element with some anchor.
This means that clusters in $\FC$ are localized around each anchor, and anchors are the cores of $\FC$.
Hence, using the components $\bar{\FC}_1, \ldots, \bar{\FC}_r$ of $\FC$, given as 
\begin{align} \label{Exp: barF_j}
 \bar{\FC}_j = \{ S \in \FC : \max_{u \in S} \|\a_u - \w_j \|_1 \le 8\mu \},
\end{align}
we can write $\FC$ as 
\begin{align*}
 \FC = \bar{\FC}_1 \cup \cdots \cup \bar{\FC}_r.
\end{align*}
If anchors are far from each other; in other words, $\omega$ is large, 
then the components $\bar{\FC}_1, \ldots, \bar{\FC}_r$ are disjoint from each other.
The left of Figure \ref{Fig: Structure of F} illustrates $\FC$.

According to the observations made so far, 
it turns out that we can find a cluster belonging to one of $\bar{\FC}_1, \ldots, \bar{\FC}_r$.
A cluster of $\FC$ is obtained by using $\GC(\p)$, and it belongs to one of $\bar{\FC}_1, \ldots, \bar{\FC}_r$ 
because $\FC$ can be written as $\FC = \bar{\FC}_1 \cup \cdots \cup \bar{\FC}_r$.
Let us denote the obtained cluster by $S_1$, and assume that $S_1$ belongs to $\bar{\FC}_1$
in order to simplify the subsequent description.
By updating the point list $\p$, 
we can find a cluster belonging to one of the remaining components $\bar{\FC}_2, \ldots, \bar{\FC}_r$.
Let $\q$ be a point list made by updating $\p$ as 
\begin{align*}
 \q(u) = 
 \left\{
 \begin{array}{ll}
  0      & \text{if} \ u \in S_1, \\
  \p(u)  & \text{otherwise}.
 \end{array}
\right.
\end{align*}
We show in Lemma \ref{Lem: Representation of F(q)} that 
$\FC(\q)$ can be written as 
\begin{align*}
  \FC(\q) = \bar{\FC}_2 \cup \cdots \cup \bar{\FC}_r.
\end{align*}
The right of Figure \ref{Fig: Structure of F} illustrates $\FC(\q)$.
A cluster, denoted by $S_2$, of $\FC(\q)$ is obtained by using $\GC(\q)$, and it belongs to one of $\bar{\FC}_2, \ldots, \bar{\FC}_r$.
By repeating the procedure, we can find $r$ clusters $S_1, \ldots, S_r$ 
such that $S_j \in \bar{\FC}_j$ for each $j \in R$ by rearranging the indices of $\bar{\FC}_1, \ldots, \bar{\FC}_r$.
The obtained clusters provide near-basis columns. 
We choose one element from each cluster and construct the set $J$. 
Rearranging the columns of $\W$, we find that it satisfies
\begin{align*}
 \|\W - \A(:,J) \|_1 \le 8 \mu.
\end{align*}
This leads to Theorem \ref{Thm: Robustness of refined algo with postprocessing}.

\begin{algorithm}[t]
 \caption{Refinement of Hottopixx with postprocessing}
 \label{Algo: Refined algo with postprocessing}
 \smallskip
 Input: $\A \in \Real^{d \times n}$ and a positive integer $r$. \\
 Output: $\W_{\out} \in \Real^{d \times r}$.
 \begin{enumerate}[1.]
  \item Compute the optimal solution $\OptSol \in \Real^{n \times n}$ of problem $\OurModel(\A,r)$.

  \item Set $\p_1 = \diag(\OptSol), J = \emptyset$ and $\ell = 1$. Perform the following procedure.
	\begin{enumerate}[{2-}1.]
	 \item Find $S_\ell$ such that 
	       \begin{align*}
		S_\ell = \arg \min_{S \in  \GC(\p_\ell)} \diam(S).
	       \end{align*}
	 \item Choose one element from $S_\ell$ and add it to $J$. Increase $\ell$ by $1$.
	 \item If $\ell = r$, then return $\W_{\out} = \A(:,J)$ and terminate;
	       otherwise, construct $\p_\ell \in \Real_+^n$ as
	       \begin{align*}
		\p_\ell(u) = 
		\left\{
		 \begin{array}{ll}
		  0            & \text{if} \ u \in S_1 \cup \cdots \cup S_{\ell-1}, \\
		  \p_1(u)  & \text{otherwise,}
		 \end{array}
		\right.
	       \end{align*}
	       and go to step 2-1.
	\end{enumerate}
 \end{enumerate}
\end{algorithm}

Algorithm \ref{Algo: Refined algo with postprocessing} is a formal description of our algorithm.
It takes as input $(\A, r)$.
Step 2 is the postprocessing.
The cost of step 2 is dominated by step 2-1.
The cost of step 2-1 is in turn dominated by the computation of the $L_1$ distance 
between any two columns of $\A \in \Real^{d \times n}$,
which takes $O(n^2 d)$ flops. 
We below summarize the definition and role of $T_j, \FC(\p), \GC(\p)$ and $\bar{\FC}_j$,
which are used for analyzing Algorithm \ref{Algo: Refined algo with postprocessing} in Section \ref{Sec: Analysis}.

\begin{itemize}

 \item $T_j$ is a cluster, called anchor, which is defined as in (\ref{Exp: anchor}).
       
 \item $\FC(\p)$ is a collection of clusters constructed by using a point list $\p$.
       This is formed as  $\FC(\p) = \cup_{i \in N} \FC_i(\p)$ where 
       $\FC_i(\p)$ is defined as in (\ref{Exp: Fi}).
       If $\p$ is a point list obtained from the optimal solution of problem $\OurModel$, 
       we abbreviate $\FC(\p)$ and $\FC_i(\p)$ as $\FC$ and $\FC_i$, respectively.

 \item $\GC(\p)$ is a collection of clusters constructed by using a point list $\p$.
       This is formed as  $\GC(\p) = \cup_{i \in N} \GC_i(\p)$ where 
       $\GC_i(\p)$ is defined as in (\ref{Exp: Gi}), 
       which is obtained by discarding some condition imposed on $\FC_i(\p)$.
       
 \item $\bar{\FC}_j$ is the component of $\FC$, defined as in (\ref{Exp: barF_j}).
       We show in Lemma \ref{Lem: Representation of F(p)} that $\FC$ is written as $\FC = \cup_{j \in R} \bar{\FC}_j$.

\end{itemize}

\subsection{Analysis} \label{Sec: Analysis}

\subsubsection{Scores of Anchors}
We show that anchors have high scores by using the point list obtained by solving  problem $\OurModel$.

\begin{lem} \label{Lem: Score of anchors}
 Let $\A$ satisfy Assumption \ref{Asm: Noisy separable matrix A}.
 Assume $\kappa > 0$. Let $\mu$ satisfy $\mu \neq 0$ and $\epsilon \le \mu$.
 Set $\p = \diag(\OptSol)$ for the optimal solution $\OptSol$ of  problem $\OurModel(\A, r)$.
 Then, anchors $T_1, \ldots, T_r$ with  parameter $\mu$ satisfy
 \begin{align*}
  \SCORE(T_j, \p) \ge 1 - \frac{16 \epsilon}{\kappa \mu (1 - \epsilon)}.
 \end{align*}
 for every $j \in R$.
\end{lem}
We can prove this in a similar way as Lemma \ref{Lem: Score of basis}; the proof is in Appendix B.
From Lemma \ref{Lem: Score of anchors},
we immediately obtain Corollary \ref{Cor: Score of anchors if amount of noise is small}.

\begin{cor}\label{Cor: Score of anchors if amount of noise is small}
 Let $\A$ satisfy Assumption \ref{Asm: Noisy separable matrix A}.
 Set $\p = \diag(\OptSol)$ for the optimal solution $\OptSol$ of  problem $\OurModel(\A, r)$.
 Consider two cases as follows:
\begin{itemize}
 \item Let $\epsilon$ satisfy $\epsilon < \frac{\kappa \omega}{578(r+1)}$. 
       The value of $\mu$ is set as $\mu = \frac{17(r+1)\epsilon}{\kappa} + \xi$ 
       by choosing an arbitrary real number $\xi$ from the open interval $(0, \frac{\kappa}{35})$.

 \item Let $\epsilon$ satisfy $\epsilon < \frac{\kappa^2}{289(r+1)^2}$. 
       The value of $\mu$ is set as $\mu = \sqrt{\epsilon} + \xi$ 
       by choosing an arbitrary real number $\xi$ from the open interval $(0, \frac{\kappa}{35})$.
\end{itemize}
 The following hold in both cases.
 \begin{enumerate}[{\normalfont (a)}]
  \item $0 \le \epsilon < 1$.
  \item $0 < \mu < \frac{\omega}{17}$.
  \item $\epsilon \le \mu$.
  \item	$\SCORE(T_j, \p) > \frac{r}{r+1}$ for every $j \in R$.
 \end{enumerate}
\end{cor}
We can easily check that the corollary holds. The proof is given in Appendix C.
Part (a) just tells us that the bounds imposed on $\epsilon$ in the two cases
do not violate Assumption \ref{Asm: Noisy separable matrix A}(b).
The role of $\xi$ is to prevent the value of $\mu$ from being zero; hence,
we are allowed to choose an arbitrary real number from the open interval $(0, \frac{\kappa}{35})$.

\subsubsection{Structure of $\FC$}
We prove the observations about $\FC$ that we made in Section \ref{Subsec: Algorithm of postprocessing}.

\begin{lem} \label{Lem: F contains hatS}
 Let $\FC(\q) \neq \emptyset$ for some $\q \in \Real_+^n$. 
 Then, $\GC(\q) \neq \emptyset$.
 Moreover, $\FC(\q)$ contains $\hat{S} = \arg \min_{S \in \GC(\q)} \diam(S)$.
\end{lem}
\begin{proof}
 Since $\FC_i(\q) \subset \GC_i(\q) \subset \GC(\q)$ and $\FC(\q) = \bigcup_{i \in N} \FC_i(\q)$,
 any element of $\FC(\q)$ belongs to $\GC(\q)$. 
 Hence, $\FC(\q) \subset \GC(\q)$ holds.
 Consequently, $\FC(\q) \neq \emptyset$ implies $\GC(\q) \neq \emptyset$.

 Since $\hat{S}$ belongs to $\GC(\q) = \bigcup_{i \in N} \GC_i(\q)$, 
 we have $\hat{S} \in \Omega_{i_*}$ for some $i_* \in N$ and $\SCORE(\hat{S}, \q) > \frac{r}{r+1}$.
 From the relation $\FC(\q) \subset \GC(\q)$, we have 
 \begin{align*}
  \diam(\hat{S}) = \min_{S \in \GC(\q)} \diam(S) \le \min_{S \in \FC(\q)} \diam(S) \le 3 \mu.
 \end{align*}
 Consequently, $\hat{S} \in \FC_{i_*}(\q)$, which implies $\hat{S} \in \FC(\q)$.
\end{proof}

\begin{lem} \label{Lem: F contains all anchors}
 Frame the hypotheses of Corollary \ref{Cor: Score of anchors if amount of noise is small}.
 The following hold:
 \begin{enumerate}[{\normalfont (a)}]
  \item Anchor $T_j$ is not empty.
  \item All anchors $T_1, \ldots, T_r$ belong to $\FC$.
  \item Anchor $T_j$ belongs to the component $\bar{\FC}_j$ of $\FC$.
 \end{enumerate}
\end{lem}
\begin{proof} 
 Separability means that there is a map $\phi : R \rightarrow N$ such that $\w_j = \v_{\phi(j)}$  for each $j \in R$. 
 We use the map $\phi$ in the proof of parts (a) and (b).

 (a) From Corollary \ref{Cor: Score of anchors if amount of noise is small}(c),
 we have 
 \begin{align*}
  \| \a_{\phi(j)} - \w_j \|_1 = \|\v_{\phi(j)} + \n_{\phi(j)} - \w_j \|_1 = \|\n_{\phi(j)}\|_1  \le \epsilon  \le \mu.
 \end{align*}
 Hence, $T_j$ contains $\phi(j)$, which means that $T_j$ is not empty.

 (b) We show that $T_j$ belongs to $\FC_{\phi(j)}$ for each $j \in R$.
 Since $\phi(j) \in T_j$, as shown in part (a), we have $T_j \in \Omega_{\phi(j)}$.
 By Corollary \ref{Cor: Score of anchors if amount of noise is small}(d), 
 the score of $T_j$ by $\p$ satisfies $\SCORE(T_j, \p) > \frac{r}{r+1}$.
 The diameter of $T_j$ in $\Omega_{\phi(j)}$ satisfies $\diam(T_j) \le 3 \mu$, 
 since any $u \in T_j$ satisfies 
 \begin{align*}
  \|\a_u - \a_{\phi(j)} \|_1 = \|\a_u - \v_{\phi(j)}  - \n_{\phi(j)}\|_1 = \|\a_u - \w_j  - \n_{\phi(j)}\|_1
  \le \|\a_u - \w_j \|_1 + \|n_{\phi(j)}\|_1 
  &\le 2\mu + \epsilon  \\
  &\le 3 \mu.
 \end{align*}
 The last inequality uses Corollary \ref{Cor: Score of anchors if amount of noise is small}(c).
 Hence, $T_j \in \FC_{\phi(j)}$ for each $j \in R$.
 In addition, the definition of $\FC$ implies $\FC_{\phi(j)} \subset \FC$.
 Consequently, $T_1, \ldots, T_r$ belong to $\FC$.

 (c) We have already shown $T_j \in \FC$ for each $j \in R$ in part (b).
 The definition of $T_j$ implies that, for any $u \in T_j$, we have $\|\a_u - \w_j \|_1 \le 2 \mu \le 8\mu$.
 Hence, $T_j$ belongs to $\bar{\FC}_j$.
\end{proof}
Parts (a) and (c) tell us that the components $\bar{\FC}_1, \ldots, \bar{\FC}_r$ of $\FC$ are not empty.
We will use this observation in the proof of Theorem \ref{Thm: Robustness of refined algo with postprocessing}.

\begin{lem}\label{Lem: Cluster in F intersects with some of anchors}
 Frame the hypotheses of Corollary \ref{Cor: Score of anchors if amount of noise is small}.
 For any $S \in \FC$, there is some $j \in R$ such that $S \cap T_j \neq \emptyset$.
\end{lem}
\begin{proof}
 We start by showing that any two different anchors do not have a common element.
 Let $x, y \in R$ and $x \neq y$.  No $u \in T_x$ belongs to $T_y$, since 
 \begin{align*}
  \| \a_u - \w_y \|_1   
  & =  \| (\w_x - \w_y) + (\a_u - \w_x)  \|_1     & \\
  & \ge \| \w_x - \w_y \|_1 - \|\a_u - \w_x \|_1  & \\
  & \ge \omega - 2\mu   & \text{(by the definition of $\omega$)} \\ 
  & > 15 \mu            & \text{(by Corollary \ref{Cor: Score of anchors if amount of noise is small}(b))}.
 \end{align*}
 Hence, $T_x \cap T_y = \emptyset$ holds for any different $x$ and $y$ in $R$.
 We will prove the lemma by contradiction. 
 Assume that there is some $S \in \FC$ such that $S \cap T_j = \emptyset$ for any $j \in R$.
 Since $T_x \cap T_y = \emptyset$ for  $x, y \in R$ with $x \neq y$,
 any two different clusters among $S, T_1, \ldots, T_r$ do not have a common element.
 Hence, we have
 \begin{align*}
  \SCORE(S, \p) + \sum_{j \in R} \SCORE(T_j, \p) = \SCORE(S \cup T_1 \cup \cdots \cup T_r, \p) \le \SCORE(N, \p) = r.
 \end{align*}
 The last equality follows from the fact that $\SCORE(N, \p) = \trace(\OptSol) = r$ holds
 since the first constraint of problem $\OurModel$ requires $\OptSol$ to satisfy $\trace(\OptSol) = r$.
 By Corollary \ref{Cor: Score of anchors if amount of noise is small}(d), 
 the score of $T_j$ by $\p$ satisfies $\SCORE(T_j, \p)  > \frac{r}{r+1}$ for each $j \in R$.
 Therefore, we get $\SCORE(S, \p) \le \frac{r}{r+1}$ and reach a contradiction to $S \in \FC$, 
 which means $\SCORE(S, \p)  > \frac{r}{r+1}$.
 The assumption is false. That is, for any $S \in \FC$, there is some $j \in R$ such that $S \cap T_j \neq \emptyset$.
\end{proof}

\begin{lem} \label{Lem: Representation of F(p)}
 Frame the hypotheses of Corollary \ref{Cor: Score of anchors if amount of noise is small}.
 The following hold:
 \begin{enumerate}[{\normalfont (a)}]
  \item $\FC$, i.e., the abbreviation of $\FC(\p)$, is represented as 
	\begin{align*}
	 \FC = \bigcup_{j \in R} \bar{\FC}_j
	\end{align*}
	by using the components $\bar{\FC}_j$ of $\FC$.
  \item Let $x, y\in R$ and $x \neq y$.
	We have $S_x \cap S_y = \emptyset$ for any $(S_x, S_y) \in \bar{\FC}_x \times \bar{\FC}_y$.
 \end{enumerate}
\end{lem}
 \begin{proof}
  (a) First, we prove the inclusion ``$\supset$''. 
  Let $S \in \cup_{j \in R} \bar{\FC}_j$. Then, there is a $j_* \in R$ such that $S \in \bar{\FC}_{j_*}$.
  The definition of $\bar{\FC}_{j_*}$ implies $S \in \FC$.
  Hence, the inclusion ``$\supset$'' holds.
  
  Next, we prove the inclusion ``$\subset$''.
  Let $S \in \FC$.  Recall that $\FC$ is defined by $\FC = \cup_{i \in N} \FC_i$.
  Hence, there is an $i_* \in N$ such that $S \in \FC_{i_*}$.
  Lemma \ref{Lem: Cluster in F intersects with some of anchors} ensures that 
  there is a $j_* \in R$ such that $S \cap T_{j_*}  \neq \emptyset$.
  Let $v \in S \cap T_{j_*}$. Then, for any $u \in S$,   
  \begin{align*}
   \|\a_u - \w_{j_*} \|_1 
   = \|(\a_u - \a_v)+ (\a_v - \w_{j_*}) \|_1 
   &\le \| \a_u - \a_v \|_1 + \| \a_v - \w_{j_*} \|_1              & \\
   &\le \| \a_u - \a_v \|_1 + 2\mu                                & \text{(by $v \in T_{j_*}$)} \\
   &= \| (\a_u - \a_{i_*}) + (\a_{i_*} - \a_v) \|_1 + 2\mu        & \\
   &\le \| \a_u - \a_{i_*} \|_1 + \| \a_{i_*} - \a_v \|_1 + 2\mu  & \\
   &\le 8\mu.                                                     & \text{(by $u,v \in S$ and $S \in \FC_{i_*}$)}
  \end{align*}
  Accordingly, we have $S \in \bar{\FC}_{j_*}$ for $j_* \in R$,
  which implies $S \in \cup_{j \in R} \bar{\FC}_j$. 
  Hence, the inclusion ``$\subset$'' holds.
  Consequently, $\FC = \bigcup_{j \in R} \bar{\FC}_j$ as claimed.

  (b) Let $x, y \in R$ and $x \neq y$.
  Let $S_x \in \bar{\FC}_x$ and $S_y \in \bar{\FC}_y$. We have, for any $u \in S_x$, 
  \begin{align*}
   \| \a_u - \w_y \|_1 
   & = \| (\w_x - \w_y) + (\a_u - \w_x) \|_1  \\
   & \ge \| \w_x - \w_y \|_1 - \| \a_u - \w_x \|_1  \\
   & \ge \omega -  8\mu   & \text{(by the definition of $\omega$ and $S_x \in \bar{\FC}_x$)} \\
   & > 9\mu               & \text{(by Corollary \ref{Cor: Score of anchors if amount of noise is small}(b))}.
  \end{align*}
  Hence, $u \notin S_y$. This means $S_x \cap S_y = \emptyset$.
 \end{proof}

Here, we prove Lemma \ref{Lem: Property of the union of disjoint sets} for establishing Lemma \ref{Lem: Representation of F(q)}.
In Lemmas \ref{Lem: Property of the union of disjoint sets} and \ref{Lem: Representation of F(q)}, we use the following notation: 
$\ell_1, \ldots, \ell_r$ denote the $r$ integers in $R$;
$k$ is any positive integer satisfying $k < r$; 
and $K$ is the set of consecutive integers from $1$ to $k$.

\begin{lem} \label{Lem: Property of the union of disjoint sets}
 Frame the hypotheses of Corollary \ref{Cor: Score of anchors if amount of noise is small}.
 We have the relation 
 \begin{align*}
  \bigcup_{j \in K} \bar{\FC}_{\ell_j} = \FC \setminus \bigcup_{j \in R \setminus K} \bar{\FC}_{\ell_j}.
 \end{align*}
\end{lem}
\begin{proof}
 Lemma \ref{Lem: Representation of F(p)}(b) implies $\bar{\FC}_x \cap \bar{\FC}_y = \emptyset$ 
 for $x, y \in R$ with $x \neq y$. We  use this relation in the proof.
 To simplify the description, 
 we denote $\AC = \bigcup_{j \in K} \bar{\FC}_{\ell_j}$ and $\BC = \bigcup_{j \in R \setminus K} \bar{\FC}_{\ell_j}$.

 First, we prove the inclusion ``$\subset$''. 
 Let $S \in \AC$. Then, $S \in \bar{\FC}_{\ell_{j_*}}$ for some $j_* \in K$.
 This implies $S \in \FC$ by the definition of $\bar{\FC}_{\ell_{j_*}}$.
 In addition, as shown above, we have $\bar{\FC}_{\ell_{j_*}} \cap \bar{\FC}_{\ell_j} = \emptyset$ for every $j \in R \setminus K$.
 Hence, $S \notin \bigcup_{j \in R \setminus K} \bar{\FC}_{\ell_j}$. Consequently, the inclusion ``$\subset$'' holds.

 Next, we prove  the inclusion ``$\supset$''.
 It holds if $\FC \setminus \BC = \emptyset$.
 In what follows, we thus assume $\FC \setminus \BC \neq \emptyset$. We use contradiction.
 Since the assumption means that there exists $S \in \FC \setminus \BC$, we choose such $S$.
 Let us assume contradiction; $S \notin \AC$. 
 Then, $S \in \FC \cap \AC^c \cap \BC^c$.
 Meanwhile,
 \begin{align*}
  \FC \cap \AC^c \cap \BC^c = \FC \cap ( \AC \cup \BC )^c 
  = \FC \cap \left(  \bigcup_{j \in R} \bar{\FC}_{\ell_j} \right)^c =  \FC \cap \FC^c  = \emptyset
 \end{align*}
 holds by De Morgan's laws and Lemma \ref{Lem: Representation of F(p)}(a).
 This contradicts the fact that $S$ exists.
 Hence, the inclusion ``$\supset$'' holds.
 Consequently,  $\bigcup_{j \in K} \bar{\FC}_{\ell_j} = \FC \setminus \bigcup_{j \in R \setminus K} \bar{\FC}_{\ell_j}$ as claimed.
\end{proof}

\begin{lem} \label{Lem: Representation of F(q)}
 Frame the hypotheses of Corollary \ref{Cor: Score of anchors if amount of noise is small}.
 Let $S_1, \ldots, S_k$ be clusters such that $S_j \in \bar{\FC}_{\ell_j}$ for each $j \in K$. 
 Suppose that we are given $S_1, \ldots, S_k$ and the point list $\p$.
 Construct a point list $\q \in \Real_+^n$:
 \begin{align*}
  \q(u) = 
  \left\{
  \begin{array}{ll}
   0      & \text{if} \ u \in S_1 \cup \cdots \cup S_k, \\
   \p(u)  & \text{otherwise}.
  \end{array}
  \right.
 \end{align*}
 Then, the following hold:
 \begin{enumerate}[{\normalfont (a)}]
  \item Let $S \in \FC$. Then,
	\begin{align*}
	 S \in \bigcup_{j \in R \setminus K} \bar{\FC}_{\ell_j}
	 \equivSym
	 \SCORE(S, \q) > \frac{r}{r+1}.
	\end{align*}
  
  \item $\FC(\q)$ is represented as 
	\begin{align*}
	 \FC(\q) =  \bigcup_{j \in R \setminus K} \bar{\FC}_{\ell_j}.
	\end{align*}
 \end{enumerate}
\end{lem}
\begin{proof}
 (a) First, we prove the direction ``$\Rightarrow$''.
 Let $S \in \bigcup_{j \in R \setminus K} \bar{\FC}_{\ell_j}$. 
 Then, $S$ belongs to $\bar{\FC}_{\ell_{j_*}}$ for some $j_* \in R \setminus K$.
 Meanwhile, $S_j$ belongs to $\bar{\FC}_{\ell_j}$ for $j \in K$.
 Lemma \ref{Lem: Representation of F(p)}(b) then tells us that $S \cap S_j = \emptyset$ for every $j \in K$.
 Hence, from the construction of $\q$, we have $\q(u) = \p(u)$ for every $u \in S$. 
 In addition, 
 since $S \in \bar{\FC}_{\ell_{j_*}}$ implies $S \in \FC$ by the definition of $\bar{\FC}_{\ell_{j_*}}$,
 it takes $\SCORE(S, \p) > \frac{r}{r+1}$.
 Consequently, we obtain $\SCORE(S, \q) = \SCORE(S, \p) > \frac{r}{r+1}$.

 Next, we prove the direction ``$\Leftarrow$'' by showing that the contrapositive is true.
 In light of Lemma \ref{Lem: Property of the union of disjoint sets},  the contrapositive statement is 
 \begin{align} \label{Exp: Contrapositive statement}
  S \in \bigcup_{j \in K} \bar{\FC}_{\ell_j} \Rightarrow \SCORE(S, \q) \le \frac{r}{r+1}.
 \end{align}
 Let $S \in \bigcup_{j \in K} \bar{\FC}_{\ell_j}$.
 Then, $S$ belongs to $\bar{\FC}_{\ell_{j_*}}$ for some $j_* \in K$.
 From the construction of $\q$, we have $\q(u) = 0$ for every $u \in S_{j_*}$. 
 Hence, 
 \begin{align} \label{Exp: score(S,q)}
  \SCORE(S, \q) = \sum_{u \in S} \q(u) = \sum_{u \in \bar{S}} \q(u).
 \end{align}
 for $\bar{S} = S \setminus S_{j_*}$.
 Let $S_{k+1}, \ldots, S_r$ be clusters such that $S_j \in \bar{\FC}_{\ell_j}$ for each $j \in R \setminus K$.
 Since $S_j \in \bar{\FC}_{\ell_j}$ for $j \in R$ and $\bar{S} = S \setminus S_{j_*}$ where $S, S_{j_*} \in \bar{\FC}_{\ell_{j_*}}$,
 Lemma \ref{Lem: Representation of F(p)}(b) tells us that the following statements hold:
 \begin{align}
  & \bar{S} \cap S_j = \emptyset \quad \text{for every} \ j \in R. \label{Exp: barS cap Sj} \\
  & S_x \cap S_y = \emptyset \quad \text{for every different} \ x, y \in R. \label{Exp: Sx cap Sy}
 \end{align}
 Statement (\ref{Exp: barS cap Sj}) implies that no element of $\bar{S}$ belongs to $S_1 \cup \dots \cup S_k$.
 Hence,
 \begin{align} \label{Exp: score(barS, p)} 
  \sum_{u \in \bar{S}} \q(u) = \sum_{u \in \bar{S}} \p(u) = \SCORE(\bar{S}, \p).
 \end{align}
 It follows from equalities (\ref{Exp: score(S,q)}) and (\ref{Exp: score(barS, p)}) that the relation $\SCORE(S, \q) = \SCORE(\bar{S}, \p)$ holds.
 From statements (\ref{Exp: barS cap Sj}) and (\ref{Exp: Sx cap Sy}), we have 
 \begin{align*}
  \SCORE(\bar{S}, \p) + \sum_{j \in R} \SCORE(S_j, \p) = \SCORE(\bar{S} \cup S_1 \cup \cdots \cup S_r, \p) \le \SCORE(N, \p) = r.
 \end{align*}
 Here, $\SCORE(S_j, \p) > \frac{r}{r+1}$ since $S_j \in \bar{\FC}_{\ell_j}$ implies $S_j \in \FC$.
 Accordingly, the inequality above yields $\SCORE(\bar{S}, \p) \le \frac{r}{r+1}$.
 Combining it with the relation $\SCORE(S, \q) = \SCORE(\bar{S}, \p)$, we obtain $\SCORE(S, \q) \le \frac{r}{r+1}$.
 Consequently, statement (\ref{Exp: Contrapositive statement}) holds.

 (b) First, we prove the inclusion ``$\supset$''. 
 Let $S \in \bigcup_{j \in R \setminus K} \bar{\FC}_{\ell_j}$.
 Then, $S$ belongs to $\bar{\FC}_{\ell_{j_*}}$ for some $j_* \in R \setminus K$.
 It thus follows from part (a) that $\SCORE(S, \q) > \frac{r}{r+1}$.
 In addition, $S \in \bar{\FC}_{\ell_{j_*}}$ implies $S \in \FC$ by the definition of $\bar{\FC}_{\ell_{j_*}}$.
 From $\FC = \cup_{i \in N} \FC_i$, we have $S \in \FC_{i_*}$ for some $i_* \in N$.
 Thus, $S \in \Omega_{i_*}$ and $\diam(S) \le 3\mu$ by the definition of $\FC_{i_*}$.
 Consequently, we obtain $S \in \FC_{i_*}(\q)$, which implies $S \in \FC(\q)$ since $\FC(\q) = \cup_{i \in N}\FC_i(\q)$.

 Next, we prove the inclusion ``$\subset$''. 
 Let $S \in \FC(\q)$.  Then, $S$ belongs to $\FC_{i_*}(\q)$ for some $i_* \in N$, 
 since $\FC(\q) = \cup_{i \in N} \FC_i(\q)$.
 It follows from the definition of $\FC_{i_*}(\q)$ that 
 $S \in \Omega_{i_*}, \diam(S) \le 3 \mu$, and $\SCORE(S, \q) > \frac{r}{r+1}$.
 Since $S$ satisfies $\SCORE(S, \q) > \frac{r}{r+1}$, 
 part (a) ensures that  the inclusion ``$\subset$'' holds if $S \in \FC$.
 Thus, the remainder of the proof is to show $S \in \FC$.
 The construction of a point list $\q$ tells us that $\p(i) \ge \q(i)$ for every $i \in N$. 
 Hence, 
 \begin{align*}
  \SCORE(S, \p)  \ge \SCORE(S, \q) > \frac{r}{r+1}.
 \end{align*}
 holds.
 Consequently, we obtain $S \in \FC_{i_*}(\p)$, which implies $S \in \FC$
 since $\FC = \FC(\p) = \cup_{i \in N} \FC_i(\p)$.

\end{proof}

\subsubsection{Robustness to Noise}
We are now ready to prove Theorem \ref{Thm: Robustness of refined algo with postprocessing}.

\begin{proof}[(Proof of Theorem \ref{Thm: Robustness of refined algo with postprocessing})]
 Let $S_1, \ldots, S_r$ be clusters generated by Algorithm \ref{Algo: Refined algo with postprocessing}.
 We claim that there is a permutation $\pi : R \rightarrow R$ such that $S_\ell \in \bar{\FC}_{\pi(\ell)}$ 
 for each $\ell \in R$. We use induction on $\ell$.
 Set a parameter $\mu$ as $\mu = \lambda + \xi$ by choosing 
 an arbitrary real number $\xi$ from the open interval $(0, \frac{\kappa}{35})$.
 The value of $\lambda$ is set according to the noise level described in the theorem:
 \begin{itemize}
  \item $\lambda = \frac{17(r+1)\epsilon}{\kappa}$ in the former case where  $\epsilon < \frac{\kappa \omega}{578(r+1)}$.
  \item $\lambda = \sqrt{\epsilon}$ in the latter case where $\epsilon < \frac{\kappa^2}{289(r+1)^2}$.
 \end{itemize}

 Base case:
 Step 1 of the algorithm computes the optimal solution $\OptSol$ of problem $\OurModel(\A,r)$ and 
 step 2 sets $\p_1 = \diag(\OptSol)$.
 Thus, Lemmas \ref{Lem: F contains all anchors} and  \ref{Lem: Representation of F(p)} hold.
 Lemma \ref{Lem: Representation of F(p)}(a) tells us that $\FC(\p_1)$ is represented as $\FC(\p_1) = \bigcup_{j \in R} \bar{\FC}_j$.
 It follows from Lemmas \ref{Lem: F contains all anchors}(a) and \ref{Lem: F contains all anchors}(c) that 
 the components $\bar{\FC}_1 \ldots, \bar{\FC}_r$ are not empty.
 This means that $\FC(\p_1)$ is not empty.
 We can thus use Lemma \ref{Lem: F contains hatS}, which tells us that 
 $\FC(\p_1)$ contains $S_1 = \arg \min_{S \in  \GC(\p_1)} \diam(S)$. 
 Accordingly, there is a $j \in R$  such that $S_1 \in \bar{\FC}_j$.

 Induction step:
 Let $\ell_1, \ldots, \ell_r$ denote the $r$ integers in $R$.
 Let $k$ be any positive integer satisfying $k < r$, and $K$ be the set of consecutive integers from $1$ to $k$.
 Suppose that $S_j \in \bar{\FC}_{\ell_j}$ holds for each $j \in K$.
 Lemma \ref{Lem: Representation of F(q)} holds; part (b) of the lemma tells us that 
 $\FC(\p_{k+1})$ is represented as $\FC(\p_{k+1}) = \bigcup_{j \in R \setminus K} \bar{\FC}_{\ell_j}$.
 As mentioned above,  $\bar{\FC}_{k+1} \ldots, \bar{\FC}_r$ are not empty.
 Hence, $\FC(\p_{k+1})$ is not empty.
 We can thus use Lemma \ref{Lem: F contains hatS}, which tells us that 
 $\FC(\p_{k+1})$ contains $S_{k+1} = \arg \min_{S \in  \GC(\p_{k+1})} \diam(S)$. 
 Accordingly, there is a $j \in R \setminus K$  such that $S_{k+1} \in \bar{\FC}_{\ell_j}$.
 Consequently, there is a permutation $\pi : R \rightarrow R$ such that 
 $S_\ell \in \bar{\FC}_{\pi(\ell)}$ for each $\ell \in R$.

 In light of the definition of $\bar{\FC}_j$, this result implies that the output $\W_\out = \A(:, J)$ of the algorithm 
 satisfies  
 \begin{align*}
  \| \W - \W_\out \|_1 \le 8\mu = 8(\lambda + \xi)  
 \end{align*}
 by rearranging the columns of $\W$.
 Since the inequality holds for any small positive number $\xi$,
 it turns out that 
 \begin{align*}
  \| \W - \W_\out \|_1 \le  8\lambda  
 \end{align*}
 holds.
 This gives the desired results.
\end{proof}

\section{Experiments} \label{Sec: Experiments}
We conducted experiments to see  the practical performance of our algorithms.
Gillis and Luce \cite{Gil14a} observed in their experiments that 
the postprocessing of Gillis \cite{Gil13} does not always  enhance the robustness of their refinement of Hottopixx.
For that reason, they proposed to incorporate a hybrid postprocessing into their refinement.
A detailed description was given in Algorithm 6 of \cite{Gil14a}.
They implemented it and showed its superiority to other algorithms.
We incorporated the algorithmic framework of hybrid postprocessing into Algorithm \ref{Algo: Refined algo with postprocessing},
as described in Algorithm \ref{Algo: Hybrid}, and implemented it on MATLAB.
The purpose of our experiments was to demonstrate its performance.

We compared four algorithms as follows: RHHP (Algorithm \ref{Algo: Hybrid}), 
LP-rho1 (Algorithm 6 of \cite{Gil14a}), Hottopixx (Algorithm 1 of \cite{Gil14a}) 
and SPA (Algorithm 1 with $f(\x) = \|\x\|_2^2 $ of \cite{Gil14b}).
SPA was originally proposed in \cite{Ara01} in the context of chemometrics,
and is now considered a popular algorithm for solving separable NMF problems.
For the implementation of LP-rho1, Hottopixx and SPA, we used the MATLAB functions  
{\tt LPsepNMF\_cplex}, {\tt hottopixx\_cplex} and {\tt FastSepNMF} 
whose code is available at the website of the first author of \cite{Gil14a}.
For solving LP problems, the functions {\tt hottopixx\_cplex} and {\tt LPsepNMF\_cplex} employed CPLEX.
Following them, we  employed it in the implementation of RHHP.

\begin{algorithm}[t]
 \caption{Refinement of Hottopixx with hybrid postprocessing}  \label{Algo: Hybrid}
 \smallskip
 Input: $\A \in \Real^{d \times n}$ and a positive integer $r$. \\
 Output: Set $J$ of $r$ elements from $N$.
 \begin{enumerate}[1.]
  \item Perform step 1 of Algorithm \ref{Algo: Refined algo with postprocessing}.
	Let $J_1$ be the index set corresponding to the $r$ largest elements of $\diag(\OptSol)$.

  \item Perform step 2 of Algorithm \ref{Algo: Refined algo with postprocessing} where
	step 2-2 chooses 
	\begin{align*}
	 u = \arg \max_{u \in S_\ell} \p_{\ell}(u).
	\end{align*}
	Let $J_2 = J$ for the index set $J$ obtained at the termination of step 2. 

  \item Compute 
	\begin{align*}
	 J = \arg \min_{J \in \{ J_1, J_2 \}} \error(J) 
	 \quad \mbox{where} \quad  \error(J) = \min_{\X \ge \zero} \| \A - \A(:, J) \X \|_F^2
	\end{align*}
	and return $J$.
 \end{enumerate}
\end{algorithm}

We tested the algorithms on four synthetic datasets whose construction is the same as in \cite{Gil14a, Gil14c}.
Each dataset contained noisy separable matrices  $\A = \W\H + \N \in \Real^{30 \times 200}$ 
where the factorization rank is $10$ and the set of basis indices is $\{1, \ldots, 10\}$.
The components $\W \in \Real^{30 \times 10}_+, \H \in \Real^{10 \times 200}_{+}$ and $\N \in \Real^{30 \times 200}$ 
were generated as follows.
\begin{itemize}
\item $\W$:
      Using the following procedures (A) and (B),  two types of matrices were generated.
      \begin{enumerate}[(A)]
       \item 
	     Normal:
	     First, generate $\W \in \Real^{30 \times 10}$ whose elements are drawn 
	     from a uniform distribution on the interval $[0,1]$.
	     Then, normalize the columns to have unit $L_1$ norm.
       \item 
	     Ill-conditioned:
	     First, generate $\W \in \Real^{30 \times 10}$ as in the first step of procedure above.
	     Second, compute the reduced SVD $\W = \F \Sigmab \G^\trans$ 
	     where $\Sigmab$ is a diagonal matrix of size $10$, $\F \in \Real^{30 \times 10}$, and $\G \in \Real^{10 \times 10}$.
	     Third, choose a positive integer $c$ 
	     and replace $\W$ by $\F \S \G^\trans$ using a diagonal matrix $\S$ of size $10$ whose 
	     $i$th diagonal element is $\alpha^{(i-1)}$ for $\alpha \in \Real$ satisfying $\alpha^9 = 10^{-c}$. 
	     Finally, replace all negative elements by $0$ and then 
	     normalize the columns to have unit $L_1$ norm.
      \end{enumerate}

\item $\H$:
      It is formed as $\H = [\I, \bar{\H}]$
      where the submatrix composed of $10$ columns from the first one is an identity matrix of size $10$, 
      and the columns of the remaining submatrix of size \by{10}{190} are 
      from a Dirichlet distribution whose $10$ parameters are uniformly from the interval $[0, 1]$.
      Hence, $\H$ is nonnegative and every column has unit $L_1$ norm.
      Moreover, if one constructs $\V=\W\H$, 
      our parameter choice of a Dirichlet distribution encourages
      the columns of $\V$ to lie around the boundary of the convex hull of the columns of $\W$.

 \item $\N$:
       First, choose a positive real number $\delta$ serving as a noise intensity, 
       and generate $\N \in \Real^{30 \times 200}$
       whose elements are from a standard normal distribution.
       Then, normalize it such that the $L_1$ norm is equal to $\delta$.
       Hence, the resulting matrix $\N$ satisfies $\|\N\|_1 = \delta$.
\end{itemize}
To generate noisy separable matrices in dataset 1, 
we chose 20 equally spaced points $\delta$ in log space between $10^{-2}$ and $1$, 
and then constructed $\N$ satisfying $\|\N\|_1 = \delta$ for each $\delta$. 
We used the matrices $\W$ generated by procedure (A).
For those in datasets 2-4, we chose 20 equally spaced points $\delta$ in log space between $10^{-2}$ and $0.5$, 
and then constructed $\N$ satisfying $\|\N\|_1 = \delta$.
We used ill-conditioned matrices $\W$ generated by procedure (B) with the choice of $c$ as follows:
$c=3$ for dataset 2, $c=4$ for dataset 3, and $c=5$ for dataset 4.
In the construction of datasets 1-4, we generated 50 separable matrices $\V = \W\H$, 
and then formed noisy separable matrices $\A = \V + \N$
by adding 20 matrices $\N$ to each $\V$; hence, each dataset contained 1,000 matrices in total.
Table \ref{Tab: parameter values} displays
the average values of $\kappa, \omega, \sigma_{\mmax} / \sigma_{\mmin}$ and $\beta$
over 50 matrices $\W$ and $\H$ in datasets 1-4. Here, $\sigma_{\mmax} / \sigma_{\mmin}$ is 
the ratio of the largest singular value of $\W$ divided by the smallest one.
Recall that $\kappa$ and $\omega$ are defined in terms of $\W$ and $\beta$ in terms of the submatrix $\bar{\H}$ of $\H$.

\begin{table}[t]
 \centering
 \caption{Average values of $\kappa, \omega, \sigma_{\mmax} / \sigma_{\mmin}$ and $\beta$ 
 over 50 matrices $\W$ and $\H$ in datasets 1-4.}
 \label{Tab: parameter values}
\begin{tabular}{lcccc}
 \toprule
                & Dataset 1 & Dataset 2 & Dataset 3 & Dataset 4 \\
 \midrule
 Type of $\W$   & Normal    & Ill-conditioned & Ill-conditioned & Ill-conditioned \\
                      &           &      with $c=3$ &      with $c=4$ &      with $c=5$ \\
 \midrule
 $\kappa$       & $3.27 \times 10^{-1}$  & $3.67 \times 10^{-2}$  & $1.16 \times 10^{-2}$  & $3.43 \times 10^{-3}$ \\
 $\omega$       & $4.70 \times 10^{-1}$  & $1.47 \times 10^{-1}$  & $8.62 \times 10^{-2}$  & $5.15 \times 10^{-2}$ \\
 $\sigma_{\max} / \sigma_{\mmin}$        
                & $1.09 \times 10^{1}$   &  $3.07 \times 10^{2}$  & $3.38 \times 10^{3}$   & $5.36 \times 10^{4}$ \\
 \midrule
 $\beta$        & $8.03 \times 10^{-1}$  &  $8.03 \times 10^{-1}$ & $8.03 \times 10^{-1}$  & $8.03 \times 10^{-1}$ \\
 \bottomrule
\end{tabular}
\end{table}

\begin{figure}[p]
 \centering
 \begin{minipage}[b]{0.45\linewidth}
 \includegraphics[width=1.0\linewidth]{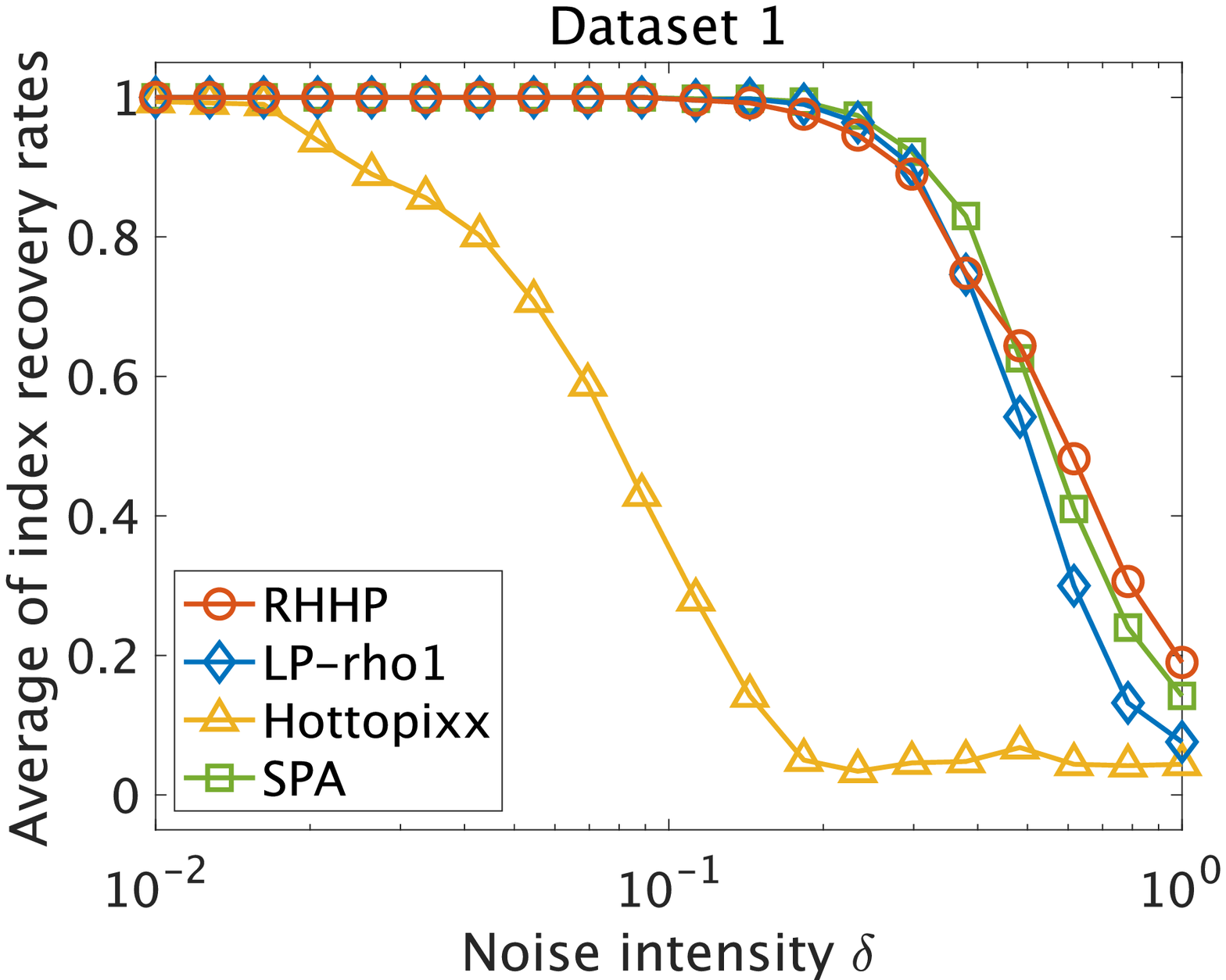}
 \end{minipage}
\begin{minipage}[b]{0.45\linewidth}
 \includegraphics[width=1.0\linewidth]{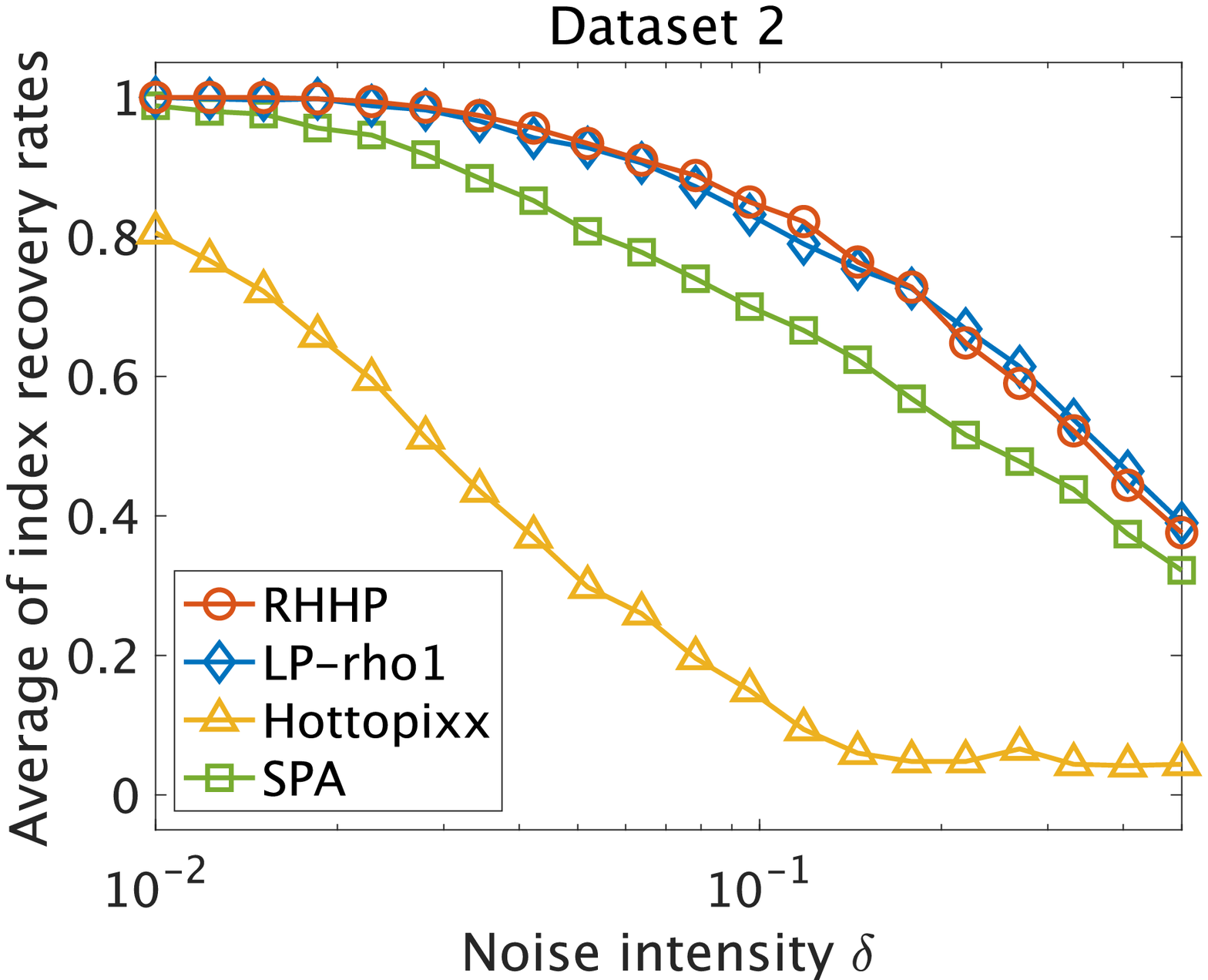}
\end{minipage}
\begin{minipage}[b]{0.45\linewidth}
 \includegraphics[width=1.0\linewidth]{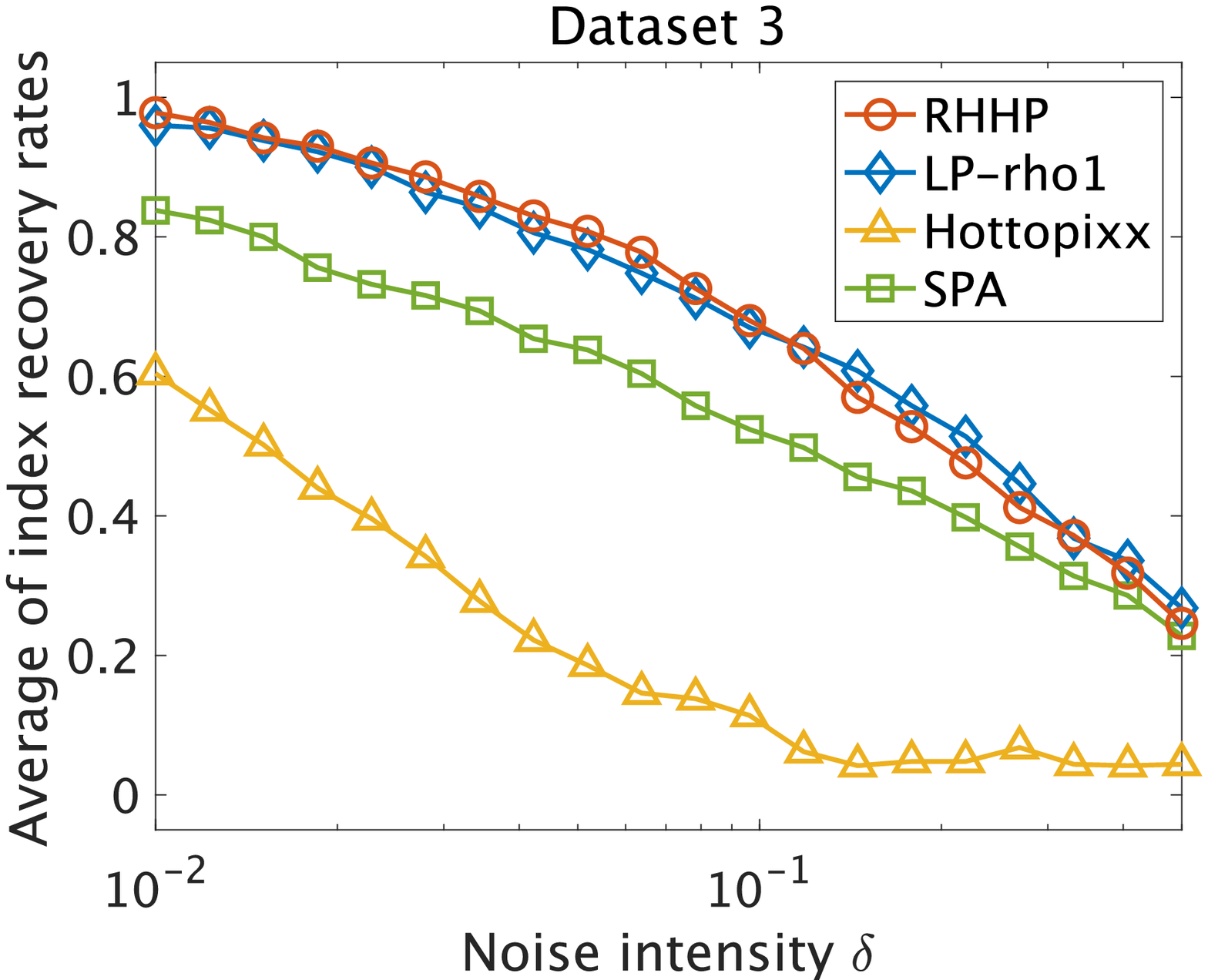}
\end{minipage}
\begin{minipage}[b]{0.45\linewidth}
 \includegraphics[width=1.0\linewidth]{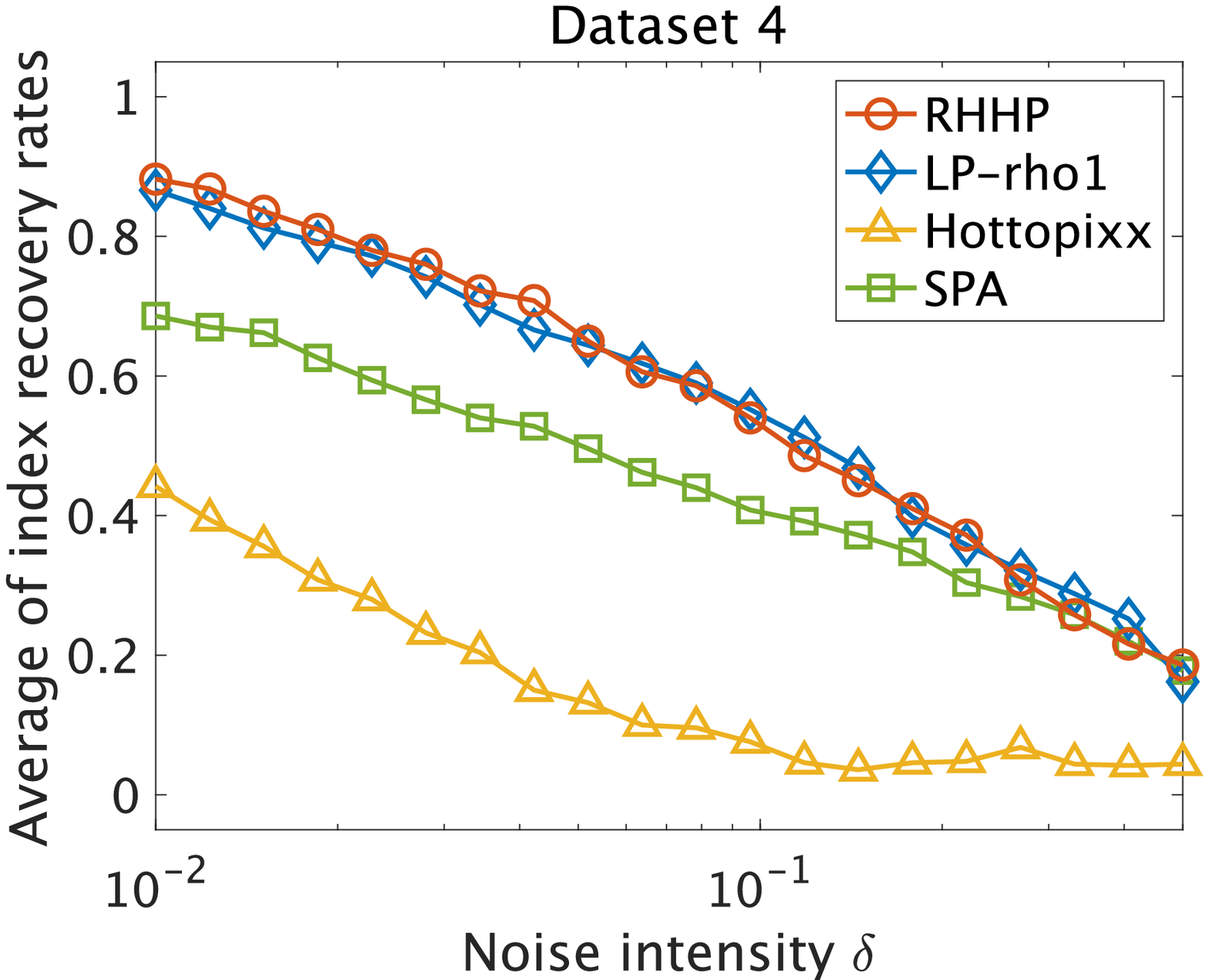}
\end{minipage}
 \caption{Average of index recovery rates by four algorithms for datasets 1-4.}
 \label{Fig: index recovery}

\bigskip \bigskip \bigskip
 \captionof{table}{Maximum values of $\delta$ for 100\% and 80\% recovery of basis indices.
 The symbol ``-'' in 100\% recovery (resp.\ 80\% recovery) 
 means that the average of the index recovery rates at $\delta = 0.01$ is less than 1 (resp.\ 0.8).
 The bold-faced values indicate the maximum value in each column.}
 \label{Tab: maximum delta}
\begin{tabular}{lcccccccc}
 \toprule
           & \multicolumn{2}{c}{Dataset 1} & \multicolumn{2}{c}{Dataset 2} & \multicolumn{2}{c}{Dataset 3} & \multicolumn{2}{c}{Dataset 4} \\
           \cmidrule(rl){2-3}  \cmidrule(rl){4-5} \cmidrule(rl){6-7} \cmidrule(rl){8-9}
           & 100\% & 80\% & 100\% & 80\%  & 100\% & 80\% & 100\% & 80\% \\
 \midrule
 RHHP      & {\bf 0.089} & 0.298       & {\bf 0.015} & {\bf 0.118} & -  & {\bf 0.052} & -  & {\bf 0.019} \\
 LP-rho1   & {\bf 0.089} & 0.298       & 0.010       & 0.096       & -  & 0.042       & -  & 0.015 \\
 Hottopixx &  -          & 0.043       & -           & 0.010       & -  & -           & -  &  -  \\
 SPA       & {\bf 0.089} & {\bf 0.379} & -           & 0.052       & -  & 0.015       & -  &  -  \\
 \bottomrule
\end{tabular}
\end{figure}

The performance of the algorithm was evaluated by using the index recovery rate, defined by 
 $|J \cup \{1, \ldots, 10\}| / 10$ for an index set $J$ output by it.
LP-rho1 and Hottopixx required us to designate a noise level $\epsilon$ as input.
For a matrix $\A = \W\H + \N$ in the datasets,
we set $\epsilon = \|\N\|_1$, which is equal to $\delta$, and then ran the algorithms.
The experiments were conducted on Intel Xeon CPU E5-1620 with 64 GB memory running MATLAB.

Figure \ref{Fig: index recovery} and Table \ref{Tab: maximum delta} summarize the experimental results:
the figure displays the average of index recovery rates determined by the four algorithms;
and the table lists the maximum values of $\delta$ for 100\% and 80\% recovery of basis indices by them.
Regarding the index recovery rates of the algorithms, we can see the following:
\begin{itemize}
 \item RHHP, LP-rho1, and SPA are better than Hottopixx for every dataset.
 \item For dataset 1, SPA is slightly better than RHHP and LP-rho1, since the maximum value of $\delta$ for 80\% recovery 
       determined by SPA exceeds those determined by RHHP and LP-rho1. RHHP is almost the same as LP-rho1.
 \item For datasets 2-4, RHHP and LP-rho1 are better than SPA. RHHP is slightly better than LP-rho1, 
       since the maximum values of $\delta$ for 80\% recovery determined by RHHP exceed those determined by LP-rho1.
\end{itemize}
The experimental results imply that, 
without taking a noise level as input, RHHP is as robust to noise as LP-rho1.

\section{Concluding Remarks}

We refined Hottopixx of Bittorf et al.\ \cite{Bit12} and the postprocessing of Gillis \cite{Gil13}
and showed that our refinement has almost the same robustness to noise as the original one.
To enable Hottopixx to run without  prior knowledge of the noise level,
we replaced the problem $\BRRT$ with $\OurModel$.
This is a simple idea, and it is easy to see that Lemma \ref{Lem: Relation of epsilon and optimal value} holds.
From the lemma,
we can immediately see that the refinement is similar in robustness to Hottopixx.
However, it is not obvious how the postprocessing of Gillis can be refined 
so that the algorithm runs without prior knowledge of the noise level.
We constructed a collection $\FC$ of clusters containing anchors $T_1, \ldots, T_r$ and examined the structure of $\FC$.
On the basis of this examination,
we developed a refinement of the postprocessing and analyzed its robustness to noise.

We close this paper with remarks on directions for future research.
There is a computational issue in Algorithms \ref{Algo: Refined algo} and \ref{Algo: Refined algo with postprocessing}.
The bottleneck is in solving problem $\OurModel$.
As shown in Section \ref{Subsec: Algorithm of refined Hottopixx},
this can be transformed into an equivalent LP problem $\OurModel'$ with $O(n^2)$ variables and $O(n^2)$ constraints
where $n$ is the number of columns of the input matrix and we assume that it is greater than the number $d$ of rows.
Since the size of $\OurModel'$ grows quadratically with $n$, 
solving $\OurModel'$ is computationally challenging when $n$ is large.
We thus need to develop efficient algorithms.
Bittorf et al.\ \cite{Bit12} and Gillis and Luce \cite{Gil18} used first-order methods 
and developed algorithms for solving their optimization models $\BRRT$ and $\GL$.
The use of first-order methods would be promising for solving $\OurModel'$ efficiently.

Regarding the bounds given 
in Theorems \ref{Thm: Robustness of refined algo} and \ref{Thm: Robustness of refined algo with postprocessing},
it remains to investigate the tightness of them.
Recently, Gillis \cite{Gil19} studied an ideal algorithm for solving separable NMF problems.
Since the computational cost grows exponentially with the problem size,
it is not realistic to apply the algorithm to large problems.
They showed that it achieves the best possible bound on the error relative to the basis.
There is a gap between our error bound 
shown for Algorithm \ref{Algo: Refined algo with postprocessing} in Theorem \ref{Thm: Robustness of refined algo with postprocessing}
and the optimal one.
It would be interesting to see whether we can reduce the gap.

\appendix
\section*{Appendix A \quad Proof of Lemma \ref{Lem: Properties of optimal solution}}

\begin{proof}[(Proof of Lemma \ref{Lem: Properties of optimal solution})]
 We prove the first inequality.
 By Lemma \ref{Lem: Relation of epsilon and optimal value}, 
\begin{align*}
 2 \epsilon \ge \OptVal =  \| \A \OptSol - \A\|_1 
 &\ge \|\A \OptSol(:,i) - \a_i \|_1 \\
 &\ge \|\A \OptSol(:,i) \|_1 - \| \a_i \|_1 \\
 &\ge \|\V \OptSol(:,i) + \N \OptSol(:,i)\|_1 - \| \a_i \|_1 \\
 &\ge \underbrace{\|\V \OptSol(:,i)\|_1}_{\mathrm{(A)}} 
 - \underbrace{\| \N \OptSol(:,i)\|_1}_{\mathrm{(B)}}
 - \underbrace{\| \a_i \|_1}_{\mathrm{(C)}}.
\end{align*}
 The term $\mathrm{(A)}$ can be rewritten as
\begin{align*}
 \mathrm{(A)} = \one^\trans \V\OptSol(:,i) = \one^\trans \OptSol(:,i) = \| \OptSol(:,i) \|_1
\end{align*}
 since $\one^\trans \V = \one$ by Assumption~\ref{Asm: Noisy separable matrix A}(a) and $\V, \OptSol \ge \zero$.
 By using Assumptions \ref{Asm: Noisy separable matrix A}(a) and \ref{Asm: Noisy separable matrix A}(b),
 we bound the terms $\mathrm{(B)}$ and $\mathrm{(C)}$ as follows:
\begin{align*}
 \mathrm{(B)} &\le \|\N\|_1 \|\OptSol(:,i)\|_1 \le \epsilon \|\OptSol(:,i)\|_1, \\
 \mathrm{(C)} &=  \| \v_i + \n_i \|_1 \le \| \v_i \|_1 + \|\n_i\|_1 \le 1 + \epsilon. 
\end{align*}
 Hence, we obtain $1 + 3\epsilon  \ge (1 - \epsilon) \|\OptSol(:,i)\|_1$, 
 which gives the first inequality of this lemma,
 since $0 \le \epsilon < 1$ by Assumption~\ref{Asm: Noisy separable matrix A}(b).
 
 Next, we prove the second inequality. 
 By Lemma \ref{Lem: Relation of epsilon and optimal value}, 
\begin{align*}
 2 \epsilon \ge \OptVal = \|\A - \A\OptSol\|_1 
 &\ge \|\a_i - \A\OptSol(:,i)\|_1  \\
 &= \|\v_i + \n_i - \V\OptSol(:,i) - \N\OptSol(:,i)\|_1  \\
 &=   \|\v_i - \V\OptSol(:,i)  + \n_i - \N\OptSol(:,i)\|_1  \\
 &\ge \|\v_i - \V\OptSol(:,i)\|_1  - \| \n_i - \N\OptSol(:,i) \|_1  \\
 &\ge \|\v_i - \V\OptSol(:,i)\|_1  - \underbrace{(\| \n_i \|_1 + \|\N\OptSol(:,i) \|_1)}_{\mathrm{(A)}}.
\end{align*}
 By using Assumption~\ref{Asm: Noisy separable matrix A}(b) and the first inequality of this lemma, 
 we bound the term $\mathrm{(A)}$ as follows:
 \begin{align*}
 \mathrm{(A)} 
 \le \| \n_i \|_1 + \|\N\|_1 \| \OptSol(:,i) \|_1
 \le \frac{2 \epsilon (1 + \epsilon)}{1 - \epsilon}.
 \end{align*}
 We then obtain the second inequality of this lemma.

\end{proof}

\section*{Appendix B \quad Proof of Lemmas \ref{Lem: Score of basis} and \ref{Lem: Score of anchors}}

We use the following lemma to prove Lemmas \ref{Lem: Score of basis} and \ref{Lem: Score of anchors}.

\begin{lem} \label{Lem: Key lemma for evaluating scores}
 Let $\A$ satisfy Assumption \ref{Asm: Noisy separable matrix A}.
 Let $\phi: R \rightarrow N$ be a map such that  $\w_j = \v_{\phi(j)}$ for each $j \in R$.
 Then, for $j \in R$ and $i = \phi(j) \in N$, we have
 \begin{align*}
  \| (1 - \eta + \tilde{\epsilon}) \w_j - \W(:, R \setminus \{j\}) \z \|_1 \le 2 \tilde{\epsilon} \quad \mbox{and} \quad
   1 - \eta + \tilde{\epsilon} \ge 0
 \end{align*}
 by letting 
 \begin{align*}
  \eta             = \H(j,:) \OptSol(:,i),  \quad 
  \z               = \H(R \setminus \{j\}, :) \OptSol(:, i) \quad \mbox{and} \quad
  \tilde{\epsilon} = \frac{4\epsilon}{1 - \epsilon}
 \end{align*}
 where $\OptSol$ is the optimal solution of problem $\OurModel(\A,r)$.
\end{lem}
\begin{proof}
 Let $j \in R$ and $i = \phi(j) \in N$.
 Lemma \ref{Lem: Properties of optimal solution} tells us that 
 \begin{align*}
  \tilde{\epsilon} \ge \|\v_i - \V \OptSol(:, i) \|_1  
  &= \|\v_i - \W\H \OptSol(:, i) \|_1 \\
  &= \|\v_{\phi(j)} - \W\H \OptSol(:, i) \|_1 \\
  &= \|\w_j - \underbrace{\W\H \OptSol(:, i)}_{\mathrm{(A)}} \|_1.
 \end{align*}
 Since 
 \begin{align*}
  \W\H  = \w_1 \H(1, :) + \cdots + \w_r \H(r, :)  = \w_j \H(j,:) + \W(:, R \setminus \{j\})  \H(R \setminus \{j\}, :)  
 \end{align*}
 the term $\mathrm{(A)}$ is rewritten as 
 \begin{align*}
  \mathrm{(A)} = \eta \cdot \w_j + \W(:, R \setminus \{j\}) \z.
 \end{align*}
 by letting 
 \begin{align*}
  \eta = \H(j, :)\OptSol(:, i) \in \Real 
  \quad \mbox{and} \quad 
  \z = \H(R \setminus \{j\}, :)\OptSol(:, i) \in \Real^{r-1}.
 \end{align*}
 Accordingly,
 \begin{align*}
  \tilde{\epsilon}  
  &\ge \|(1-\eta)\w_j - \W(:, R \setminus \{j\})\z \|_1 
  & \\
  &= \|(1-\eta+\tilde{\epsilon})\w_j - \W(:, R \setminus \{j\})\z - \tilde{\epsilon}\w_j\|_1 
  & \\
  &\ge \|(1-\eta+\tilde{\epsilon})\w_j - \W(:, R \setminus \{j\})\z \|_1 - \tilde{\epsilon} \|\w_j\|_1  
  & \\
  &= \|(1-\eta+\tilde{\epsilon})\w_j - \W(:, R \setminus \{j\})\z \|_1 - \tilde{\epsilon}    
  & \text{(by Assumption~\ref{Asm: Noisy separable matrix A}(a)).}
 \end{align*}
 Note that $\|\tilde{\epsilon} \w_j\|_1 = \tilde{\epsilon}\|\w_j\|_1$ holds
 since $\tilde{\epsilon} = \frac{4\epsilon}{1 - \epsilon} \ge 0$ by Assumption~\ref{Asm: Noisy separable matrix A}(b).
 We thus obtain 
 $\| (1 - \eta + \tilde{\epsilon}) \w_j - \W(:, R \setminus \{j\}) \z \|_1 \le 2 \tilde{\epsilon}$
 for $\eta$ and $\z$ defined above. 
 Moreover, considering that all elements of $\H$ are less than or equal to $1$ 
 since Assumption~\ref{Asm: Noisy separable matrix A}(a) holds and
 $\H \ge \zero$, we have 
 \begin{align*}
  \eta = \H(j, :)\OptSol(:, i)    
  &\le \one^{\trans} \OptSol(:,i)  & \\
  &= \|\OptSol(:,i)\|_1            & \text{(by $\OptSol \ge \zero$)} \\
  &\le 1 + \tilde{\epsilon}        & \text{(by Lemma \ref{Lem: Properties of optimal solution}).}
 \end{align*}
 This gives $1 - \eta + \tilde{\epsilon} \ge 0$.
\end{proof}

We are now able to prove Lemma \ref{Lem: Score of basis}.
\begin{proof}[(Proof of Lemma \ref{Lem: Score of basis})]
 Since we put Assumption \ref{Asm: Noisy separable matrix A} on $\A$,
 it can be written as $\A = \V + \N \in \Real^{d \times n}$ for $\V \in \Real_+^{d \times n}$ and $\N \in \Real^{d \times n}$.
 Since $\V$ is $r$-separable of the form $\V = \W\H = \W [\I, \bar{\H}]\Pib$ shown in (\ref{Exp: SepNMF}),
 there is a map $\phi: R \rightarrow N$ such that $\w_j = \v_{\phi(j)}$ for each $j \in R$.
 Hence, the basis index $I$ of $\V$ is given as $I = \{  \phi(1), \ldots, \phi(j)\}$.
 Let $i = \phi(j)$ for $j \in R$.
 Lemma \ref{Lem: Key lemma for evaluating scores} tells us that 
 \begin{align*}
  \| (1 - \eta + \tilde{\epsilon}) \w_j - \W(:, R \setminus \{j\}) \z \|_1 \le 2 \tilde{\epsilon} \quad \mbox{and} \quad
  1 - \eta + \tilde{\epsilon} \ge 0
 \end{align*}
 hold for 
 \begin{align*}
  \eta             = \H(j,:) \OptSol(:,i),  \quad 
  \z               = \H(R \setminus \{j\}, :) \OptSol(:, i) \quad \mbox{and} \quad
  \tilde{\epsilon} = \frac{4\epsilon}{1 - \epsilon}.
 \end{align*}

 First,  consider the case where $1 - \eta + \tilde{\epsilon} > 0$. We find that
 \begin{align*}
  2\tilde{\epsilon} 
  &\ge \| (1 - \eta + \tilde{\epsilon}) \w_j - \W(:, R \setminus \{j\}) \z \|_1  \\
  &= (1 - \eta + \tilde{\epsilon}) \| \w_j - \W(:, R \setminus \{j\}) \z' \|_1   
  & \text{(by letting $\z' = \z / (1 - \eta + \tilde{\epsilon})$)} \\
  &\ge (1 - \eta + \tilde{\epsilon})\kappa & \text{(by the definition of $\kappa$).}
 \end{align*} 
 Note that $\z' \ge 0$ since $\z = \H(R \setminus \{j\}, :) \OptSol(:, i) \ge 0$ and $1 - \eta + \tilde{\epsilon} > 0$.
 Accordingly, we obtain a lower bound on $\eta$,
 \begin{align} \label{Exp: Lower bound on eta}
  \eta \ge 1 + \frac{(\kappa - 2) \tilde{\epsilon}}{\kappa}.
 \end{align}
 We can upper bound $\eta$ using $\p(i)$.
 Since $\w_j = \v_{\phi(j)}$ and $i = \phi(j)$, we have $\H(:,i) = \e_j$, and thus $\H(j,i) = 1$.
 In light of this, we rewrite $\eta$ as
 \begin{align*}
  \eta = \H(j,:) \OptSol(:,i) 
  &= \OptSol(i,i) + \H(j, N \setminus \{i\}) \OptSol(N \setminus \{i\}, i) \\
  &= \p(i) + \underbrace{\H(j, N \setminus \{i\}) \OptSol(N \setminus \{i\}, i)}_{\mathrm{(A)}}
 \end{align*}
 and bound the term $(\mathrm{A})$ as follows:
 \begin{align*}
 \mathrm{(A)} 
  &\le \beta \cdot \one^{\trans} \OptSol(N \setminus \{i\}, i) & \text{(by the definition of $\beta$)} \\
  &= \beta (\| \OptSol(:,i) \|_1 - \OptSol(i,i) )              & \text{(by $\OptSol \ge \zero$)} \\
  &\le \beta (1 + \tilde{\epsilon} - \p(i) )                   & \text{(by Lemma \ref{Lem: Properties of optimal solution}).}
 \end{align*}
 We thus obtain an upper bound on $\eta$,
 \begin{align} \label{Exp: Upper bound on eta using p}
  \eta \le (1 - \beta) \p(i) + \beta(1+\tilde{\epsilon}).
 \end{align}
 The bounds (\ref{Exp: Lower bound on eta})  and (\ref{Exp: Upper bound on eta using p}) yield
 \begin{align*}
  1 + \frac{(\kappa - 2) \tilde{\epsilon}}{\kappa} \le (1 - \beta) \p(i) + \beta(1+\tilde{\epsilon})
  \equivSym 
  \p(i) \ge 1 + \tilde{\epsilon} - \frac{2 \tilde{\epsilon}}{\kappa(1-\beta)}.
 \end{align*}
 Assumption \ref{Asm: Noisy separable matrix A}(b) implies $\tilde{\epsilon} = 4 \epsilon / (1 - \epsilon) \ge 0$.
 Recall that $i = \phi(j)$ for $j \in R$ and $I = \{\phi(1), \ldots, \phi(r)\}$.
 Hence, from the inequality above, we obtain 
 $\p(i) \ge 1 - \frac{8 \epsilon}{\kappa(1-\beta)(1-\epsilon)}$ for every $i \in I$.

 Next, consider the case where $1-\eta + \tilde{\epsilon} = 0$.
 By inequality (\ref{Exp: Upper bound on eta using p}), we have
 \begin{align*}
  1 + \tilde{\epsilon} = \eta \le (1 - \beta) \p(i) + \beta ( 1 + \tilde{\epsilon}),
 \end{align*}
 which gives $\p(i) \ge 1 + \tilde{\epsilon}$.
 Here, $\tilde{\epsilon} \ge 0$ and $\frac{8 \epsilon}{\kappa(1-\beta)(1-\epsilon)} \ge 0$
 by Assumption \ref{Asm: Noisy separable matrix A}(b), $\kappa > 0$ and $\beta < 1$.
 We thus obtain $\p(i) \ge 1 - \frac{8 \epsilon}{\kappa(1-\beta)(1-\epsilon)}$ for every $i \in I$.
\end{proof}

\begin{remark} \label{Remark: Observation}
 In the proof above, to find a lower bound on $\eta$, 
 we have used the observation that $1 - \eta + \tilde{\epsilon}$ is positive or zero,
 which is not taken into account in the proof of Lemma 2.2 of \cite{Gil13}.
\end{remark}

Let us move on to prove Lemma \ref{Lem: Score of anchors}.
To do so, we prove the following lemma.
\begin{lem} \label{Lem: Upper bound on H(j,T_j^c)}
 Let $\A$ satisfy Assumption~\ref{Asm: Noisy separable matrix A}.
 Let $T_j$ be an anchor with parameter $\mu$ satisfying $\epsilon \le \mu$.
 Then, for $j \in R$, we have
 \begin{align*}
  \max_{u \in T_j^c} \H(j,u) < 1 - \frac{\mu}{2}.
 \end{align*}
\end{lem}
\begin{proof}
 For any $u \in T_j^c$, 
 \begin{align*}
  2\mu < \|\w_j - \a_u \|_1 
  &= \|\w_j - \W\H(:,u) - \n_u \|_1  \\
  &\le \|\w_j - \W\H(:,u) \|_1 + \|\n_u\|_1 \\
  &\le \|\w_j - \W\H(:,u) \|_1 + \epsilon \\
  &\le \|\w_j - \W\H(:,u) \|_1 + \mu.
 \end{align*}
 Hence, $\|\w_j - \W\H(:,u) \|_1 > \mu$ holds. 
 Furthermore, 
 \begin{align*}
  \mu 
  &< \|\w_j - \W\H(:,u) \|_1 \\
  &= \|\w_j - \w_j\H(j,u) - \W(:, R \setminus \{j\}) \H(R \setminus \{j\},u)  \|_1 \\
  &\le (1 - \H(j,u)) \|\w_j \|_1 + \| \W(:, R \setminus \{j\} \|_1 \| \H(R \setminus \{j\},u)  \|_1 \\
  &= 1 - \H(j,u) + \| \H(R \setminus \{j\},u)  \|_1  & \text{(by Assumption~\ref{Asm: Noisy separable matrix A}(a))} \\
  &= 1 - 2\H(j,u) + \| \H(:, u)  \|_1 & \text{(by $\H \ge \zero$)} \\ 
  &= 2 - 2\H(j,u) & \text{(by Assumption~\ref{Asm: Noisy separable matrix A}(a))}.
 \end{align*}
 Note that $\|(1 - \H(j,u))\w_j\|_1 = (1 - \H(j,u))\|\w_j\|_1$ holds since $1 - \H(j,u) \ge 0$ 
 by Assumption~\ref{Asm: Noisy separable matrix A}(a) and $\H \ge \zero$.
 It follows from the inequality above that $\max_{u \in T_j^c} \H(j,u) < 1 - \frac{\mu}{2}$ holds.
\end{proof}

In light of this, 
we can prove Lemma \ref{Lem: Score of anchors} in a similar way as Lemma \ref{Lem: Score of basis}.
The proof is almost the same, except the evaluation of the upper bound on $\eta$.

\begin{proof}[(Proof of Lemma \ref{Lem: Score of anchors})]
 We use Lemma \ref{Lem: Key lemma for evaluating scores}.
 Let $\phi: R \rightarrow N$ be a map such that $\w_j = \v_{\phi(j)}$ for each $j \in R$.
 The lemma tells us that, for $j \in R$ and $i = \phi(j) \in N$,
 we have
 \begin{align*}
  \| (1 - \eta + \tilde{\epsilon}) \w_j - \W(:, R \setminus \{j\}) \z \|_1 \le 2 \tilde{\epsilon} \quad \mbox{and} \quad
  1 - \eta + \tilde{\epsilon} \ge 0
 \end{align*}
 where $\eta$, $\z$ and $\tilde{\epsilon}$ are as shown in the lemma.

 First, consider the case where $1 - \eta + \tilde{\epsilon} > 0$. 
 As in the proof of Lemma \ref{Lem: Score of basis},
 we have $2 \tilde{\epsilon} \ge \| (1 - \eta + \tilde{\epsilon}) \w_j - \W(:, R \setminus \{j\}) \z \|_1 \ge (1 - \eta + \tilde{\epsilon}) \kappa$,
 which gives  a lower bound on $\eta$, as shown in (\ref{Exp: Lower bound on eta}).
 We can upper bound $\eta$  using $\SCORE(T_j,\p)$.  
 Write $\eta$ as 
 \begin{align*}
  \eta = \H(j,:) \OptSol(:,i) = \underbrace{\H(j,T_j) \OptSol(T_j,i)}_{\mathrm{(A)}} 
  + \underbrace{\H(j,T_j^c) \OptSol(T_j^c,i)}_{\mathrm{(B)}}.
 \end{align*}
 The term $\mathrm{(A)}$ is bounded as follows: 
 \begin{align*}
  \mathrm{(A)} 
  &\le \one^{\trans} \OptSol(T_j, i)  
  & \text{(by Assumption~\ref{Asm: Noisy separable matrix A}(a) and $\H \ge \zero$)} \\
  &= \| \OptSol(T_j,i) \|_1           
  & \text{(by $\OptSol \ge \zero$)}.
 \end{align*}
 The term $\mathrm{(B)}$ is bounded as follows:
 \begin{align*}
  \mathrm{(B)} 
  &< \left( 1 - \frac{\mu}{2} \right) \one^{\trans} \OptSol(T_j^c,i) 
  & \text{(by Lemma \ref{Lem: Upper bound on H(j,T_j^c)})} \\
  &= \left( 1 - \frac{\mu}{2} \right) \| \OptSol(T_j^c ,i) \|_1 
  & \text{(by $\OptSol \ge \zero$)} \\
  &= \left( 1 - \frac{\mu}{2} \right) ( \| \OptSol(: ,i) \|_1 - \| \OptSol(T_j ,i) \|_1 ) 
  & \\
  &\le  \left( 1 - \frac{\mu}{2} \right) ( 1 + \tilde{\epsilon} - \| \OptSol(T_j ,i) \|_1 ) 
  & \text{(by Lemma \ref{Lem: Properties of optimal solution})}.
 \end{align*}
 We then find that
 \begin{align*}
  \eta 
  < \| \OptSol(T_j,i) \|_1 + \left( 1 - \frac{\mu}{2} \right) ( 1 + \tilde{\epsilon} - \| \OptSol(T_j ,i) \|_1 ) 
  =  \frac{\mu}{2}\| \OptSol(T_j,i) \|_1 + \left( 1 - \frac{\mu}{2} \right) ( 1 + \tilde{\epsilon}).
 \end{align*}
 Here, 
 \begin{align*}
  \| \OptSol(T_j ,i) \|_1 = \sum_{u \in T_j} \OptSol(u,i) \le \sum_{u \in T_j} \OptSol(u,u) = \SCORE(T_j, \p).
 \end{align*}
 since  $\OptSol(u,i) \le \OptSol(u,u)$ and $\OptSol \ge \zero$ by the third and fourth constraints of problem $\OurModel$.
 Accordingly, we obtain
 \begin{align} \label{Exp: Upper bound on eta using score}
  \eta < \frac{\mu}{2} \SCORE(T_j, \p) + \left( 1 - \frac{\mu}{2} \right) ( 1 + \tilde{\epsilon}). 
 \end{align}
 The bounds (\ref{Exp: Lower bound on eta})  and (\ref{Exp: Upper bound on eta using score}) yield
 \begin{align*}
  1 + \frac{(\kappa - 2) \tilde{\epsilon}}{\kappa} 
  \le \frac{\mu}{2} \SCORE(T_j, \p) + \left( 1 - \frac{\mu}{2} \right) ( 1 + \tilde{\epsilon})
  \equivSym 
  \SCORE(T_j,\p) \ge 1 + \tilde{\epsilon} - \frac{4 \tilde{\epsilon}}{\kappa \mu}. 
 \end{align*}
 Here, $\tilde{\epsilon} \ge 0$.
 We thus obtain $\SCORE(T_j,\p) \ge 1 - \frac{16 \epsilon}{\kappa \mu (1-\epsilon)}$ for every $j \in R$.

 Next, consider the case where $1-\eta + \tilde{\epsilon} = 0$.
 By inequality (\ref{Exp: Upper bound on eta using score}), we have  
 \begin{align*}
  1 + \tilde{\epsilon} = \eta < \frac{\mu}{2} \SCORE(T_j, \p) + \left( 1 - \frac{\mu}{2} \right) ( 1 + \tilde{\epsilon}),
 \end{align*}
 which gives $\SCORE(T_j, \p) > 1 + \tilde{\epsilon}$.
 Here, $\tilde{\epsilon} \ge 0$ and $\frac{16 \epsilon}{\kappa \mu (1-\epsilon)} \ge 0$.
 We thus obtain 
 $\SCORE(T_j,\p) \ge 1 - \frac{16 \epsilon}{\kappa \mu (1-\epsilon)}$ for every $j \in R$.

\end{proof}

\section*{Appendix C \quad Proof of Corollary \ref{Cor: Score of anchors if amount of noise is small}}
\begin{proof}[(Proof of Corollary \ref{Cor: Score of anchors if amount of noise is small})]
 Since Assumption \ref{Asm: Noisy separable matrix A}(a) holds,
 we have the bounds  $0 \le \kappa \le 1$ and $0 \le \omega \le 2$ shown in 
 (\ref{Exp: Bounds on kappa}) and (\ref{Exp: Bounds on omega}).
 Also, since Assumption \ref{Asm: Noisy separable matrix A}(b) holds,
 we have $\epsilon \ge 0$. Hence, the bounds imposed on $\epsilon$ in the two cases imply $\kappa > 0$.
 Accordingly, $\kappa$ and $\omega$ satisfy $0 < \kappa \le 1$ and $0 \le \omega \le 2$.
 In addition,  they satisfy the relation $\kappa \le \omega$ shown in (\ref{Exp: kappa is less than omega}).

 \fbox{Former case} \
 (a) We only have to prove $\epsilon < 1$ since $\epsilon \ge 0$ by Assumption 1(b).
 The bounds $\kappa \le 1$ and $\omega \le 2$ imply
 \begin{align} \label{Exp: Bound on epsilon in case 1}
  \epsilon < \frac{\omega \kappa}{578(r+1)} \le \frac{1}{578}.
\end{align}
 Hence, $\epsilon$ satisfies $\epsilon < 1$. 
 (b) Since $r, \kappa, \xi > 0$ and $\epsilon \ge 0$,  we have
 \begin{align*}
  \mu = \frac{17(r+1)\epsilon}{\kappa} + \xi  > 0.
 \end{align*}
 By using the bound on $\epsilon$, we can put a bound on $\mu$:
 \begin{align*}
  \mu = \frac{17(r+1)\epsilon}{\kappa} + \xi  < \frac{\omega}{34} + \xi.
 \end{align*}
 Since $\xi < \kappa / 35$ and $\kappa \le \omega$, we have
 \begin{align*}
  \mu < \frac{\omega}{34} + \frac{\kappa}{35} \le \frac{\omega}{17}.
 \end{align*}
 Hence, $0 < \mu  < \omega / 17$ holds.
 (c) Since $\xi > 0$ and $\kappa \le 1$, we have 
 \begin{align*}
  \mu = \frac{17(r+1)\epsilon}{\kappa} + \xi  > \frac{17(r+1)\epsilon}{\kappa} \ge 34 \epsilon.
 \end{align*}
 Thus, $\epsilon \le \mu$ holds.
 (d) The corollary satisfies the hypotheses of Lemma \ref{Lem: Score of anchors}.
 This is because Assumption 1(b) is not violated by the bound on $\epsilon$ that we put in part (a);
 $\kappa > 0$ holds, as explained at the beginning of the proof;
 and $\mu \neq 0$ and $\epsilon \le \mu$ hold, as shown in parts (b) and (c).
 Accordingly,
 \begin{align*}
  \SCORE(T_j, \p) \ge 1 - \frac{16 \epsilon}{\kappa \mu (1 - \epsilon)}
 \end{align*}
 holds for every $j \in R$.
 If $\epsilon = 0$, then, $\SCORE(T_j, \p) \ge 1 > r / (r+1)$.
 We thus assume $\epsilon > 0$.
 The bound on $\epsilon$ in (\ref{Exp: Bound on epsilon in case 1}) implies $\epsilon < 1 / 17$
 and we have 
 \begin{align*}
  \epsilon < \frac{1}{17} \equivSym \frac{16}{1-\epsilon} < 17.
 \end{align*}
 Write the value of $\mu$ as $\mu = \lambda + \xi$ by letting $\lambda = 17(r+1)\epsilon / \kappa$.
 We find that 
\begin{align*}
  \SCORE(T_j, \p) 
 \ge 1 - \frac{16 \epsilon}{\kappa \mu (1 - \epsilon)}  
 >  1 - \frac{17 \epsilon}{\kappa \mu}    
 > 1 - \frac{17 \epsilon}{\kappa \lambda} 
 =   \frac{r}{r+1}
 \end{align*}
 where the third inequality uses $\mu = \lambda + \xi$ and $\lambda, \xi > 0$, and 
 the equality uses  $\lambda = 17(r+1)\epsilon / \kappa$.

 \fbox{Latter case} \
 (a) As mentioned in the former case, we only have to prove $\epsilon < 1$.
 From the bound $\kappa \le 1$ and $289 = 17^2$, we obtain a bound on $\epsilon$, 
 \begin{align} \label{Exp: Bound on epsilon in case 2}
 \epsilon <  \frac{\kappa^2}{17^2(r+1)^2}  \le \left( \frac{1}{34} \right)^2.
 \end{align}
 Hence, $\epsilon$ satisfies $\epsilon < 1$.
 (b) Since $\xi > 0$ and $\epsilon \ge 0$, we have $\mu = \sqrt{\epsilon} + \xi > 0$.
 Here, $\xi$ satisfies $\xi < \kappa / 35 < \kappa / 34$,
 and we have $\kappa \le \omega$.
 Hence, using the bound on $\epsilon$, we obtain a bound on $\mu$,
\begin{align*}
 \mu = \sqrt{\epsilon} + \xi < \frac{\kappa}{17(r+1)} + \xi < \frac{\kappa}{17} \le \frac{\omega}{17}.
\end{align*}
 Hence, $0 < \mu < \omega/17$ holds.
 (c) Two functions $f_1(x) = x$ and $f_2(x) = \sqrt{x}$ satisfy $f_1(x) \le f_2(x)$ for $0 \le x \le 1$.
 Since $\xi$ satisfies $\xi > 0$ and $\epsilon$ satisfies $0 \le \epsilon < 1$ as shown in part (a),
 we have $\epsilon \le \mu = \sqrt{\epsilon} + \xi$. 
 (d) The bound on $\epsilon$ in (\ref{Exp: Bound on epsilon in case 2}) implies $\epsilon < 1/17$.
 We can thus prove this part in the same way as part (d) of the former case.

\end{proof}

\section*{Acknowledgments}
The author would like thank Nicolas Gillis of 
University of Mons who provided feedback on this manuscript, 
and thank the anonymous referees  for careful reading and helpful comments 
that enhanced the quality of this paper significantly.
This research was supported by the Japan Society for 
the Promotion of Science (JSPS KAKENHI Grant Number 20K11951).

\bibliographystyle{abbrv}
\bibliography{reference}

\end{document}